%% file: main.tex
\documentclass[10pt,twocolumn,letterpaper]{article}
\usepackage[pagenumbers]{cvpr} 
\input{preamble}

\usepackage{colortbl}
\usepackage{multirow}
\usepackage{tabularx}
\usepackage{pifont}
\usepackage{algorithm}
\usepackage{array}
\usepackage{caption}
\usepackage{adjustbox}
\usepackage{makecell}
\usepackage{algpseudocode}
\usepackage[table]{xcolor}
\definecolor{purple}{RGB}{220,220,250}
\definecolor{aogreen}{rgb}{0.0, 0.5, 0.0}

\definecolor{ExpA}{RGB}{225,247,250} 
\definecolor{ExpB}{RGB}{238,248,241} 
\definecolor{ExpC}{RGB}{255,245,230} 
\definecolor{ExpD}{RGB}{245,238,250} 
\definecolor{ExpE}{RGB}{252,235,238} 

\definecolor{ExpAText}{RGB}{0,135,160}
\definecolor{ExpBText}{RGB}{45,120,70}
\definecolor{ExpCText}{RGB}{180,105,25}
\definecolor{ExpDText}{RGB}{120,70,155}
\definecolor{ExpEText}{RGB}{190,30,55}

\definecolor{lightblue}{RGB}{215, 238, 255}
\definecolor{lightred}{RGB}{255, 225, 230}
\definecolor{improve}{HTML}{147A57}
\definecolor{regress}{HTML}{C8102E}
\definecolor{metabg}{HTML}{F7F7FC}
\definecolor{tablecolor}{HTML}{EAEAF4}
\definecolor{baselineband}{HTML}{EDEDED}
\definecolor{defaultband}{HTML}{E8F1FB}
\newcommand{\imp}[1]{\,\textsuperscript{\textcolor{improve}{\footnotesize\textbf{#1}{\boldmath$\times$}}}}
\newcommand{\dgood}[1]{\textcolor{improve}{\boldmath$#1$}}
\newcommand{\dbad}[1]{\textcolor{regress}{$#1$}}
\newlength{\maxbarwidth}

\newcommand{\neutdiff}[1]{\space\textcolor{gray}{\footnotesize(+0.0)}}

\newcommand{\cmark}{\ding{51}}%
 
\newcommand{\tokdrop}[1]{\textcolor{aogreen}{\footnotesize\,$\downarrow$#1\% tok}}
\newcommand{\graydiff}[1]{\,\textcolor{gray}{\footnotesize(#1)}}
\newcommand{\graydifftext}[1]{\,\textcolor{gray}{#1}}
\definecolor{metablue}{HTML}{0064E0}
\colorlet{covbarcol}{metablue!18}
\newlength{\covbarmax}
\newcommand{\covbar}[2]{%
  \rlap{\textcolor{covbarcol}{\rule[-0.55ex]{\dimexpr#1\covbarmax/100\relax}{2.5ex}}}%
  \hspace{4pt}#2}

\usepackage{amsthm}
\newtheorem{theorem}{Theorem}
\newtheorem{proposition}{Proposition}

\newtheorem{lemma}{Lemma}
\newcommand{\xmark}{\ding{55}}
\usepackage[most]{tcolorbox}
\newtcolorbox{callout}[1][]{%
  enhanced, breakable,
  left=2mm,right=2mm,top=2mm,bottom=2mm,
  boxrule=1pt,
  colback=metabg!80!black,
  colback=metabg, 
  arc=10pt,
  before skip=15pt,
  grow to left by=3pt,
  grow to right by=3pt,
  #1
}

\definecolor{coverpurple}{rgb}{0.471,0.471,0.690}
\usepackage[pagebackref,breaklinks,colorlinks,allcolors=coverpurple]{hyperref}
\hypersetup{
  pdftitle={CoVeR: Coverage-Based Token Pruning for Multi-View 3D Reasoning in VLMs},
  pdfauthor={Nhat-Tan Bui, Varshini Elangovan, Arun Reddy Anugu, Sreyas Mohan, Wei Ye, Dilin Wang, JQ Huang, Rakesh Ranjan, Aviral Chharia, Fernando De la Torre},
  pdfsubject={Multi-view 3D reasoning in vision-language models via coverage-based visual token pruning},
  pdfkeywords={visual token pruning, token reduction, multi-view 3D reasoning, vision-language models, VLM, 3D scene understanding, spatial coverage, training-free, plug-and-play, spatial reasoning, situated reasoning, embodied question answering, ScanQA, SQA3D, OpenEQA},
  pdfcreator={LaTeX}
}

\title{CoVeR: Coverage-Based Token Pruning for Multi-View 3D Reasoning in VLMs}

\author{
Nhat-Tan Bui$^{1,*}$ \quad
Varshini Elangovan$^{1,*}$ \quad
Arun Reddy Anugu$^{1}$ \quad
Sreyas Mohan$^{2}$ \quad
Wei Ye$^{2}$ \\
Dilin Wang$^{2}$ \quad
JQ Huang$^{2}$ \quad
Rakesh Ranjan$^{2}$ \quad
Aviral Chharia$^{1,\dagger}$ \quad
Fernando De la Torre$^{1}$ \\[5pt]
$^1$Carnegie Mellon University \quad
$^2$Meta Reality Labs\\[5pt]
\url{https://humansensinglab.github.io/CoVeR}
}
\begin{document}
\twocolumn[{
\renewcommand\twocolumn[1][]{#1}
\maketitle
\vspace{-3.25em}
\begin{center}
    \captionsetup{type=figure}
    \centering
    \includegraphics[width=\textwidth]{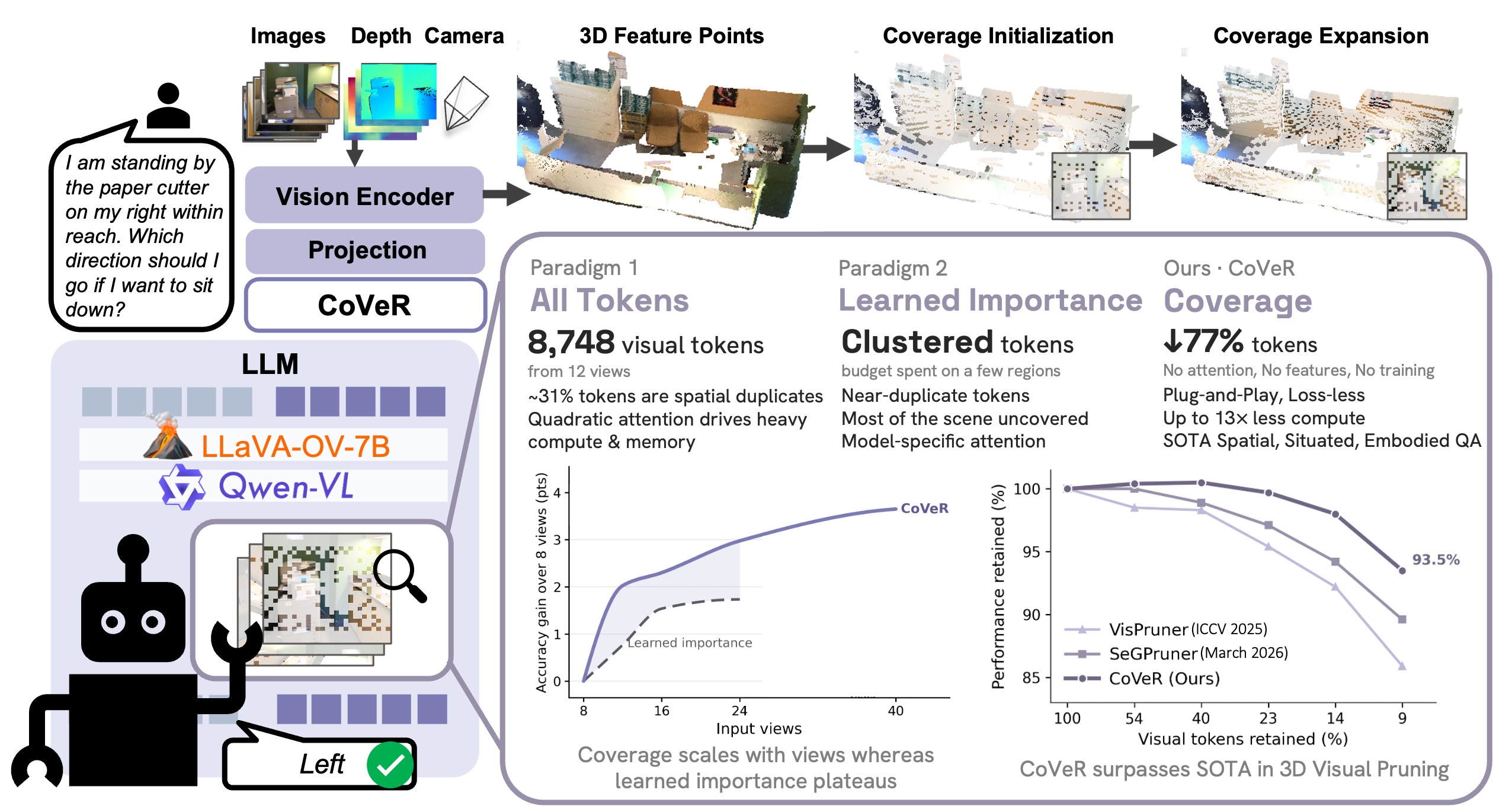}
    \caption{\textbf{Overview of CoVeR.} A deterministic, training-free token selector for multi-view 3D reasoning in 2D VLMs that prunes by spatial coverage, using geometry alone. We show that CoVeR surpasses state-of-the-art token pruning methods in the multi-view 3D setting.}
    \label{fig:teasor}
    \vspace{0.5em}
\end{center}}]
\blfootnote{*Equal Contribution\  $^\dagger$Corresponding Author}
\input{sec/0_abstract}
\input{sec/1_intro}

\input{sec/2_related}

\input{sec/3_method}
\input{sec/4_experiment}
\input{sec/5_conclusion}
{
    \small
    \bibliographystyle{ieeenat_fullname}
    \bibliography{main}
}
\thispagestyle{empty}
\maketitlesupplementary
\appendix
\setcounter{figure}{0}
\setcounter{table}{0}
\setcounter{equation}{0}
\renewcommand{\thefigure}{A\arabic{figure}}
\renewcommand{\thetable}{A\arabic{table}}
\renewcommand{\theequation}{A\arabic{equation}}
\input{sec/6_appendix}

\end{document}

%% file: preamble.tex
\newcommand{\blfootnote}[1]{%
  \begingroup
    \renewcommand\thefootnote{}%
    \footnotetext{#1}%
  \endgroup}

\newcommand{\kw}[1]{\textcolor{coverpurple}{#1}}
\newcommand{\kwb}[1]{\textcolor{coverpurple}{\textbf{#1}}}
\usepackage{microtype}

\renewcommand{\paragraph}[1]{\vspace{.5em}\noindent\textbf{#1.}}

\usepackage{xspace}

\usepackage{soul}
\setuldepth{foobar}

%% file: sec/0_abstract.tex
\begin{abstract}
Representing a 3D scene as multi-view images allows 2D VLMs to reason in 3D by reusing priors from pre-training, sidestepping the scarcity of annotated 3D data. However, it produces thousands of redundant visual tokens whose cost grows with every view. Existing visual token pruners fall into two families, each limited in the 3D multi-view setting. \textcolor{ExpDText}{Learned importance} methods rank tokens by attention or encoder features; because redundancy here is fundamentally spatial, they keep near-duplicate tokens from a few prominent regions and leave most of the scene unrepresented. \textcolor{ExpCText}{Voxelization} methods improve spatial coverage but cannot enforce an exact token budget and saturate as multi-view observations overlap in 3D, capping retention well below the target. We show that spatial coverage is associated with 3D reasoning performance and introduce \textbf{CoVeR}, a deterministic, training-free selector that uses only token coordinates, with no learned signals. CoVeR selects tokens that collectively cover every region of the scene, and solves the limitations of both families: it enforces an exact per-scene budget, breaks the \textcolor{ExpCText}{voxelization} saturation plateau, and avoids the near-duplicate selections of \textcolor{ExpDText}{learned importance}. Extensive experiments show CoVeR outperforms prior SOTAs on all \textbf{three} 3D reasoning benchmarks and generalizes as a plug-and-play module tested across \textbf{four} VLMs. Notably, with only \kw{$\approx$8\%} of visual tokens, it preserves \kw{93.5\%} of full-token performance, surpassing SOTA by \kw{3.9} percentage points on average across benchmarks.
\vspace{-1em}
\end{abstract}

%% file: sec/1_intro.tex
\section{Introduction}
\label{sec:intro}
Reasoning about 3D space is a prerequisite for systems that perceive and act in the physical world. Directly training models at scale on 3D representation~\cite{3dllm,chatscene,leo,huang20253d,xu20253d,scenellm,guo2023point,chen2024ll3da,xu2024pointllm,gpt4point,tang2024minigpt,tang2025more,inst3d,man2024situational,xiong20253ur} is limited by the scarcity of 3D-language data, which is orders of magnitude smaller than the internet-scale image-text datasets behind 2D VLMs~\cite{llava-onevision, qwen2.5, qwen3}. A practical alternative renders the scene as multi-view images and uses a pre-trained 2D VLM to reason over them~\cite{3drs, cdviews, gpt4scene, video3dllm, llava3d, choi2026drmv3d, hong20233d, splattalk, gwak2026cog3dmap, vlm3r}, inheriting their strong visual and language priors. However, this introduces a major bottleneck: the number of visual tokens grows linearly with the number of views. For example, 12-views produce 8,748 visual tokens in \texttt{LLaVA-OneVision-7B} \cite{llava-onevision}, substantially increasing the LLM inference cost. Reducing the visual token count is thus necessary for scaling 3D reasoning on 2D VLMs. 

\begin{figure*}
    \centering
    \includegraphics[width=\textwidth]{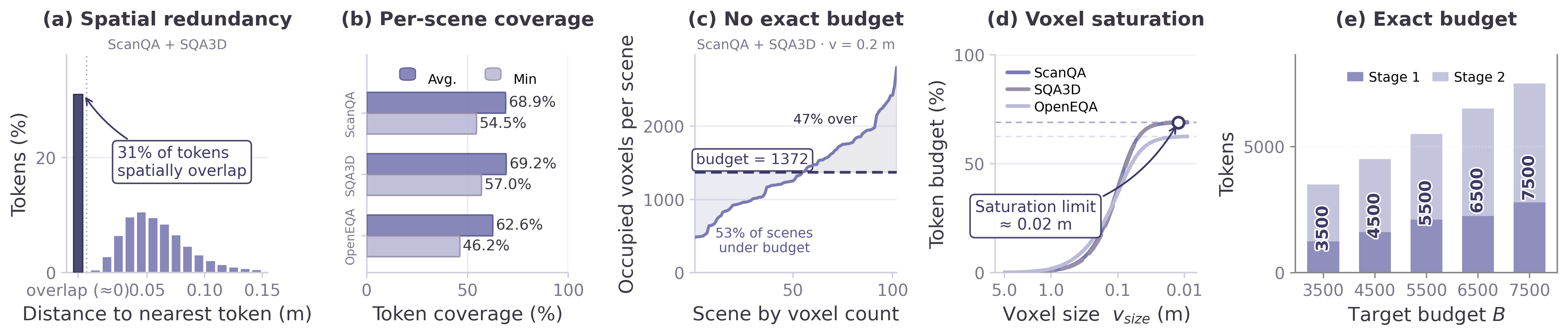}
    \caption{\textbf{Motivation (a-d) and Exact-budget Results (e).} Analysis using 12-frame inputs to \texttt{LLaVA-OneVision-7B}.}
    \label{fig:motivation}
    \vspace{-1em}
\end{figure*}

Most token pruning methods retain the top-$K$ tokens ranked by \textcolor{ExpDText}{\textit{learned importance}} from attention or visual features, an intuition inherited from single-view pruning~\cite{fastv,shang2025llava,song2025less,xing2024pyramiddrop,lin2025boosting,yang2025visionzip,zhang2024sparsevlm,bolya2022token,ye2025fit,vispruner}. We argue this is ill-suited to the multi-view 3D setting, where the dominant redundancy is geometric: different cameras observe the same physical regions, and importance scores computed independently of geometry retain near-duplicate tokens while leaving distinct objects or boundaries elsewhere under-represented. Such signals are also model-specific, depending on attention or auxiliary features that vary across architectures.

\textcolor{ExpCText}{\textit{Voxelization}}-based pruners back-project tokens into 3D space and pool those falling in the same voxels. This improves coverage but does not control the output token count: a fixed voxel size leaves a different number of tokens in different scenes, and the count plateaus once overlapping views already share voxels. For example, $\approx$31\% of tokens overlap (Fig. \ref{fig:motivation}(a)), so voxel-only pruning keeps about 69\% of tokens on average and 46-57\% in the most redundant scenes (Fig. \ref{fig:motivation}(b)). Such methods give coverage but not an exact budget.. Exact control matters because a single scene can exceed a fixed memory or latency limit even when the dataset average stays within it. Both approaches are thus insufficient in the 3D multi-view setting.

We address this with \textbf{CoVeR}, a deterministic, training-free selector that uses only token coordinates, with no attention, features, or learned signals. \textit{Coverage initialization} finds a scene-specific voxel size and keeps one representative token per occupied voxel, removing overlapping observations while preserving a coarse cover of the whole scene. \textit{Coverage expansion} then iteratively selects the tokens farthest from those retained, adding tokens in the least-covered regions until the budget is met exactly, thereby recovering fine-scale evidence beyond the saturation limit. Both stages reduce the directed Hausdorff distance from the original token set to the retained subset. CoVeR yields consistent gains at aggressive budgets and transfers across VLMs. Our contributions include:
\begin{itemize}
    \item \kwb{Problem analysis.} We identify key limitations in both 3D token pruning families: \textcolor{ExpCText}{voxelization}-based methods cannot enforce exact per-scene budgets and are capped by voxel saturation, while \textcolor{ExpDText}{learned importance}-based methods spend their budget on near-duplicates, leaving the scene under-covered.
    
    \item \kwb{Coverage-based paradigm.} We propose CoVeR, a training-free, deterministic, geometry-only token pruning framework that optimizes for scene coverage under a guaranteed exact per-scene budget.
    
    \item \kwb{Analytical insights.} Statistical metrics link geometric coverage to 3D reasoning, while directed distances show it preserves regions favored by \textcolor{ExpDText}{learned importance}; stage-wise analysis shows real tokens beat merged features and pure spatial distance outperforms learned signals.

    \item \kwb{Extensive evaluation.} CoVeR achieves SOTA on spatial scene understanding (ScanQA~\cite{scanqa}), situated reasoning (SQA3D~\cite{sqa}), and embodied question answering (OpenEQA~\cite{openeqa}), while generalizing across four VLMs. On ScanQA at 14\% retention, it cuts LLM TFLOPs by $8.6\times$ and KV cache by $7\times$, with a $1.4\times$ lower GPU memory and $2.5\times$ inference speedup at a 1.1\% relative drop.
\end{itemize}

%% file: sec/2_related.tex
\section{Related Works}
\begin{table}[t]
\centering
\caption{\textbf{Comparison with token pruning methods.} Avoiding learned signals (attention, visual features, auxiliary encoders) removes model-specific dependence, while deterministic selection and exact per-scene budget give precise control over token count.}
\label{tab:pruning_signal}
\scriptsize
\setlength{\tabcolsep}{1pt}
\resizebox{\linewidth}{!}{%
\begin{tabular}{@{}
>{\raggedright\arraybackslash}p{0.28\linewidth}
*{2}{>{\centering\arraybackslash}m{0.1\linewidth}}|
*{3}{>{\centering\arraybackslash}m{0.11\linewidth}}|
>{\centering\arraybackslash}m{0.1\linewidth}
@{}}
\toprule
& \multicolumn{2}{c|}{\scriptsize{\ \textcolor{ExpCText}{Voxelization}\ }}
& \multicolumn{3}{c|}{\scriptsize{\ \textcolor{ExpDText}{Learned Importance}\ }}
& {\scriptsize{\ Ours\ }} \\
\cline{2-3}\cline{4-6}\cline{7-7}

\textbf{Property}
& \rotatebox{65}{\shortstack{\textcolor{ExpCText}{VTC}~\cite{dtc}}}
& \rotatebox{65}{\shortstack{\textcolor{ExpCText}{DTC}~\cite{dtc}}}
& \rotatebox{65}{\shortstack{\textcolor{ExpDText}{VisPruner}~\cite{vispruner}}}
& \rotatebox{65}{\shortstack{\textcolor{ExpDText}{SeGPruner}~\cite{segpruner}}}
& \rotatebox{65}{\shortstack{\textcolor{ExpDText}{Geo3DPruner}~\cite{li2026geometry} \ }}
& \rotatebox{65}{CoVeR} \\
\midrule

Attention-free
& \cellcolor{improve!12}\textcolor{improve}{\cmark}
& \cellcolor{improve!12}\textcolor{improve}{\cmark}
& \cellcolor{regress!9}\textcolor{regress}{\xmark}
& \cellcolor{regress!9}\textcolor{regress}{\xmark}
& \cellcolor{regress!9}\textcolor{regress}{\xmark}
& \cellcolor{improve!12}\textcolor{improve}{\cmark} \\

Visual feature-free
& \cellcolor{improve!12}\textcolor{improve}{\cmark}
& \cellcolor{regress!9}\textcolor{regress}{\xmark}
& \cellcolor{regress!9}\textcolor{regress}{\xmark}
& \cellcolor{regress!9}\textcolor{regress}{\xmark}
& \cellcolor{regress!9}\textcolor{regress}{\xmark}
& \cellcolor{improve!12}\textcolor{improve}{\cmark} \\

Auxiliary encoder free
& \cellcolor{improve!12}\textcolor{improve}{\cmark}
& \cellcolor{improve!12}\textcolor{improve}{\cmark}
& \cellcolor{improve!12}\textcolor{improve}{\cmark}
& \cellcolor{improve!12}\textcolor{improve}{\cmark}
& \cellcolor{regress!9}\textcolor{regress}{\xmark}
& \cellcolor{improve!12}\textcolor{improve}{\cmark} \\

Training-free
& \cellcolor{improve!12}\textcolor{improve}{\cmark}
& \cellcolor{improve!12}\textcolor{improve}{\cmark}
& \cellcolor{improve!12}\textcolor{improve}{\cmark}
& \cellcolor{improve!12}\textcolor{improve}{\cmark}
& \cellcolor{regress!9}\textcolor{regress}{\xmark}
& \cellcolor{improve!12}\textcolor{improve}{\cmark} \\

Deterministic
& \cellcolor{improve!12}\textcolor{improve}{\cmark}
& \cellcolor{regress!9}\textcolor{regress}{\xmark}
& \cellcolor{improve!12}\textcolor{improve}{\cmark}
& \cellcolor{improve!12}\textcolor{improve}{\cmark}
& \cellcolor{improve!12}\textcolor{improve}{\cmark}
& \cellcolor{improve!12}\textcolor{improve}{\cmark} \\

Exact per-scene budget
& \cellcolor{regress!9}\textcolor{regress}{\xmark}
& \cellcolor{regress!9}\textcolor{regress}{\xmark}
& \cellcolor{improve!12}\textcolor{improve}{\cmark}
& \cellcolor{improve!12}\textcolor{improve}{\cmark}
& \cellcolor{improve!12}\textcolor{improve}{\cmark}
& \cellcolor{improve!12}\textcolor{improve}{\cmark} \\

\bottomrule
\end{tabular}%
}
\vspace{-0.5em}
\end{table}

\noindent \textbf{\textcolor{ExpCText}{\textit{Voxelization}}-based Pruning.}  These methods back-project tokens into 3D and reduce them within voxels: VTC~\cite{dtc} averages a voxel's features into a synthetic token, while DTC~\cite{dtc} raises voxel resolution and merges tokens by feature similarity. In all cases, the token count stays tied to geometry, so reported budgets are dataset averages, not exact per-scene counts. 

\noindent \textbf{\textcolor{ExpDText}{\textit{Learned importance}}-based Pruning.} These methods instead anchor on attention or visual features. SeGPruner~\cite{segpruner} combines attention-ranked initialization with diversity stage under a joint semantic-spatial metric, making its coverage attention-anchored. Geo3DPruner~\cite{li2026geometry} relies on attention and introduces a large \texttt{VGGT}~\cite{vggt} encoder, re-training the backbone. VisPruner~\cite{vispruner} ranks tokens by text-visual attention and encoder features. CoVeR differs from prior works on several axes (Table~\ref{tab:pruning_signal}): selection is purely geometric, using no attention, encoder features, or semantic similarity; it is training-free and deterministic; and it yields exact per-scene budgets that current \textcolor{ExpCText}{voxelization}-based pruners cannot guarantee.

%% file: sec/3_method.tex
\section{Methodology}
\subsection{Problem Formulation}
\label{sec:problem_formulation}
We aim to design a selector that, for a budget $B$, returns $B$ tokens whose 3D locations represent the whole observed scene, using geometry alone and no learned signals. The budget must hold on every scene rather than a dataset average, and coverage must be an explicit objective rather than a by-product of ranking. Given posed RGB-D views, a frozen visual encoder produces $M$ patch tokens, each back-projected to a world coordinate $\mathbf{t}_i \in \mathbb{R}^3$. 
Let $X=\{\mathbf{t}_i\}_{i=1}^M$ denote the token locations and let $\mathcal{C} \subseteq \{1,..., M\}$, $|\mathcal{C}|=B$ index the selected tokens. For a point $\mathbf{t}$ and an index set $\mathcal{A}$, we define $\delta(\mathbf{t;\mathcal{A}})=\min_{j\in\mathcal{A}}||\mathbf{t}-\mathbf{t}_j||_2$. We ask: given exactly $B$ retained tokens, how well do they cover the observed scene? We measure this with the directed Hausdorff distance,
\begin{equation}
    d_H(X, \mathcal{C}) = \max_{\mathbf{t} \in X} \delta(\mathbf{t};\mathcal{C})
\label{eq:hausdorff}
\end{equation}
i.e., the distance from the worst-covered token to its nearest retained token. Intuitively, $d_H$ asks how far the most under-represented part of the scene sits from anything retained, so a small value certifies that no region is dropped outright, which is precisely the guarantee \textcolor{ExpDText}{\textit{learned importance}} does not offer. Therefore, we aim
\begin{equation}
    \mathcal{C^*}=\operatorname*{arg\,min}_{\mathcal{C}\subseteq\{1,...,M\}, |\mathcal{C}|=B} d_H(X,\mathcal{C})
\end{equation}
the discrete Euclidean $k$-center problem.

\subsection{Why Voxelization Is Insufficient}
\label{sec:voxel_limits}
A natural route to the $k$-center objective is \textcolor{ExpCText}{voxelization}: at voxel size $v_s$, token $i$ has index $\mathbf{v}_i=\lfloor\mathbf{t}_i/v_s\rfloor$, occupied voxel
$k$ holds $\mathcal{V}_k=\{i\mid\mathbf{v}_i=k\}$, and keeping one token per occupied voxel returns $G(v_s;X)$ tokens, where $G(v_s;X)$ is the occupied-voxel count. This maps repeated cross-view observations to the same voxel, but its output is governed indirectly by $v_s$ rather than by a token count, which creates two problems.

\noindent\textbf{No exact per-scene budget.} Meeting budget $B$ requires $G(v_s;X)=B$, yet occupancy depends on scene geometry, extent, and view overlap, so one $v_s$ under- or overshoots $B$ across scenes. At a fixed $v_s$=0.2~m and $B$=1342, $56\%$ of scenes fall below budget and $44\%$ exceed it (Fig.~\ref{fig:motivation}(c)). VTC inherits this variability, and DTC matches retention only on average through tuning; neither guarantees a per-scene memory or latency limit.

\noindent\textbf{Voxel occupancy saturates at practical resolutions.} Reducing $v_s$ raises occupancy only up to a point: many tokens are near-duplicate observations of the same surface from different views, so past a certain resolution smaller voxels no longer separate them. About $31\%$ of ScanQA and SQA3D tokens spatially overlap (Fig.~\ref{fig:motivation}(a)), so achievable retention plateaus at $\approx$69\% near $v_s$=0.02~m (Fig.~\ref{fig:motivation}(d)); highly redundant scenes saturate at $46.2$--$57\%$ (Fig.~\ref{fig:motivation}(b)). Within the practical range, \textcolor{ExpCText}{voxelization} alone therefore cannot reach arbitrary budgets, motivating our method.

\subsection{Overview}
\label{sec:overview}
Fig.~\ref{fig:teasor} and Algo.~\ref{alg:proposedmethod} summarize CoVeR. \emph{Coverage initialization} (Sec.~\ref{sec:stage1}) adaptively voxelizes a scene and keeps an original token per occupied voxel, removing cross-view duplicates into a coarse scene-wide spatial cover. \textit{Coverage expansion} (Sec~\ref{sec:stage2}), then repeatedly adds the token farthest from the current selection, extending coverage to regions underrepresented by \textcolor{ExpCText}{voxelization}, until exactly $B$ tokens remain. The two stages are complementary. \textcolor{ExpCText}{Voxelization} is cheap and spreads tokens over the whole scene at once, but its output count is capped by resolution; farthest point sampling (FPS)~\citep{fps} can place a variable number of tokens, but builds a spread slowly. 

\noindent\textbf{Budget split.} A single hyperparameter $\alpha\in(0,1)$ sets the stage 1 target $B_{\mathrm{init}}=\max(1,\lfloor\alpha B\rfloor)$; stage 2 then supplies the remaining $B_{\mathrm{expan}}=B-|\mathcal{C}_{\mathrm{init}}|$ tokens. For any feasible budget $B\le M$, three cases ensure exactly $B$ output tokens: (1) if $|\mathcal{C}_{\mathrm{init}}|<B$, stage 2 adds the remaining tokens (i.e., $B_{\mathrm{expan}}\ge1$), this is the regime observed for all evaluated scenes and budgets; (2) if $|\mathcal{C}_{\mathrm{init}}|=B$, stage 2 is empty; and (3) if $|\mathcal{C}_{\mathrm{init}}|>B$, a safeguard (Alg.~\ref{alg:proposedmethod}) retains representatives from the $B$ most populated voxels (which is not activated at any scene or budget with our ratio $\alpha$). Thus, regardless of where the voxel search terminates, the selector always returns exactly $B$ tokens (Fig.~\ref{fig:motivation}(e)).

\noindent \textbf{Integration with a VLM.} CoVeR is a plug-in module that selects token indices $\mathcal{C}$ right after the visual encoder. All visual tokens still pass through the projector, and after projection, only the features indexed by $\mathcal{C}$ reach the LLM, kept in original sequence order. The native token layout is left intact with pruned tokens simply removed, making CoVeR compatible across diverse VLMs whose projectors preserve token-wise correspondence.

\begin{algorithm}[h]
\caption{Coverage-based 3D Token Pruning}
\label{alg:proposedmethod}
\kwb{Inputs:} Token coordinates $X=\{\mathbf{t}_i\}_{i=1}^{M}$, budget $B$, ratio $\alpha$\\
\kwb{Constants:} $\tau=0.05$, $T=16$, $[v_{\min},v_{\max}]=[0.02,5.0]$\\
\kwb{Output:} Selected token set $\mathcal{C}$, $|\mathcal{C}|=B$
\begin{algorithmic}[1]
\setlength{\itemsep}{-0.25pt}
\State $B_{\mathrm{init}} \gets \max(1,\lfloor\alpha B\rfloor)$
\Statex \textcolor{teal}{\kwb{Stage 1: Coverage initialization}}
\State $v_{\mathrm{lo}}\gets v_{\min},\;v_{\mathrm{hi}}\gets v_{\max}$
\For{$t=1,\ldots,T$}
    \State $v_s\gets(v_{\mathrm{lo}}+v_{\mathrm{hi}})/2$
    \State $G\gets|\{\lfloor\mathbf{t}_i/v_s\rfloor\}_{i=1}^{M}|$
    \State \textbf{if} $(1-\tau)B_{\mathrm{init}}\!\leq\! G\!\leq\!(1+\tau)B_{\mathrm{init}}$ \textbf{then break}
    \State $v_{\mathrm{lo}}\gets v_s$ \textbf{if} $G>(1+\tau)B_{\mathrm{init}}$ \textbf{else} $v_{\mathrm{hi}}\gets v_s$
\EndFor
\State $\mathcal{C}_{\mathrm{init}}\gets$ representative per occupied voxel \kw{\Comment{Eq.~\ref{eq:medoid}}}
\State \textbf{if} $|\mathcal{C}_{\mathrm{init}}|>B$ \textbf{then}
       $\mathcal{C}_{\mathrm{init}}\gets$ representatives of $B$ most populated voxels
       \kw{\Comment{budget safeguard}}
\Statex \textcolor{teal}{\kwb{Stage 2: Coverage expansion}}
\State $B_{\mathrm{expan}}\gets B-|\mathcal{C}_{\mathrm{init}}|,\quad \mathcal{C}_{\mathrm{expan}}\gets\emptyset$
\State $d_{\min}(m)\gets\min_{s\in\mathcal{C}_{\mathrm{init}}}\|\mathbf{t}_m-\mathbf{t}_s\|_2^2\ \ \forall m\notin\mathcal{C}_{\mathrm{init}};\quad d_{\min}(s)\gets -1\ \ \forall s\in\mathcal{C}_{\mathrm{init}}$
\For{$j=1,\ldots,B_{\mathrm{expan}}$}
    \State $m^*\gets\arg\max_m d_{\min}(m)$
    \State $\mathcal{C}_{\mathrm{expan}}\gets\mathcal{C}_{\mathrm{expan}}\cup\{m^*\},\quad d_{\min}(m^*)\gets -1$
    \State $d_{\min}(m)\gets\min(d_{\min}(m),\|\mathbf{t}_m-\mathbf{t}_{m^*}\|_2^2)\ \ \forall m$ 
\EndFor
\State \Return $\mathcal{C}\gets\mathcal{C}_{\mathrm{init}}\cup\mathcal{C}_{\mathrm{expan}}$
\end{algorithmic}
\end{algorithm}
\vspace{-0.5em}

\subsection{Coverage Initialization}
\label{sec:stage1}
Stage 1 constructs a coarse scene-wide cover. As $G(v_{s}; X)$ varies with scene geometry, CoVeR estimates $v_{s}$ per scene rather than transferring one global value across the dataset.

\noindent\textbf{Adaptive voxel size via heuristic interval search.}
Because voxel grids at different sizes are not nested, $G(v_s;X)$ is not guaranteed to be monotonic in $v_s$. However, decreasing $v_s$ tends to increase $G(v_s;X)$ (Fig.~\ref{fig:motivation}(d)). Motivated by this observation, we use binary search 
as a heuristic over an interval $[v_{\text{lo}}, v_{\text{hi}}]$, \footnote{Bounds $[v_{\min}, v_{\max}]=[0.02, 5.0]$ are fixed once by the scale of indoor scenes: below $0.02$~m occupancy no longer grows (Fig.~\ref{fig:motivation}(d)), and $5$~m yields below $0.1\%$ retention, so the interval spans the full practical range. Any $T \ge 8$ is sufficient over this range.},
raising $v_s$ when too many voxels are produced and lowering it when too few are, until $G$ lands within a tolerance $\tau$ of $B_{\text{init}}$ or $T$ iterations are reached. Because the search is heuristic, the terminal occupancy may fall on
either side of $B_{\text{init}}$, the budget safeguard (Alg.~\ref{alg:proposedmethod}) 
makes the final selection independent of this outcome. 

\noindent\textbf{Representative token per voxel.} To ensure one token per occupied voxel, we retain the token nearest the mean of the others in its voxel, which sits near the geometric center of the voxel's occupancy and is the most representative proxy. A boundary token would risk placing neighboring voxels' representatives close together, undermining uniformity:
\begin{equation}
i_k^{\text{rep}} = \arg\min_{i \in \mathcal{V}_k} \left\| \mathbf{t}_i - \frac{1}{|\mathcal{V}_k|-1}
\sum_{\substack{j \in \mathcal{V}_k,\, j \neq i}} \mathbf{t}_j \right\|_2
\label{eq:medoid}
\end{equation}
The output of this stage is the selected set:
\begin{equation} 
|\mathcal{C}_{init}| = \min\!\left(G(v_s;X),\, B\right)
\label{eq:stage1set}
\end{equation}
\subsection{Coverage Expansion}
\label{sec:stage2}
Since $G(v_{s};X)$ is capped by the number of distinct token locations, stage 1 alone cannot reach the full budget. Stage 2 lifts the ceiling by expanding $\mathcal{C}_{\text{init}}$ with an expansion set $\mathcal{C}_{\text{expan}}$ of size $B_{\text{expan}}$ via farthest point sampling (FPS)~\citep{fps}, an inherently coverage-seeking procedure. FPS iteratively builds the expansion set by selecting candidate tokens that are maximally distant from all currently selected tokens. Starting from an empty set $\mathcal{C}_{\text{expan}}$, it treats running selection as $\mathcal{C} \;=\; \mathcal{C}_{init} \cup \mathcal{C}_{\text{expan}}$. At each iteration, new token $s^{*}$ maximizes its spatial distance to nearest token already in $\mathcal{C}$:
\begin{equation}
m^{*} = \arg\max_{m \notin \mathcal{C}} \left( \min_{s \in \mathcal{C}} \mathcal{D}(\mathbf{t}_m, \mathbf{t}_s) \right)
\end{equation}
Intuitively, this adds each new token $m^{*}$ to $\mathcal{C}_{\text{expan}}$, i.e., wherever the scene is currently least covered, and repeats until the budget is met, ensuring $|\mathcal{C}| = B$.

\noindent\textbf{Voxel-initialized expansion.} Because $\mathcal{C}$ is initialized with $\mathcal{C}_{\text{init}}$, the first expansion step measures distance to the already covered regions, allowing FPS to select tokens in uncovered areas rather than re-covering regions explored in stage 1. Concretely, for every unselected token $m \notin \mathcal{C}_{\text{init}}$, the minimum distance to the initial set is computed as $d_{\min}(m) = \min_{s \in \mathcal{C}_{init}} \mathcal{D}(\mathbf{t}_m, \mathbf{t}_s) \ \forall \, m \notin \mathcal{C}_{init}$.
At each step, the token with the largest $d_{\min}$ is added to $\mathcal{C}_{\text{expan}}$ and the distances are updated against the new token, for $B_{\text{expan}}$ steps, yielding $|C|=B$. 

\noindent\textbf{Distance metric.}
We use the squared Euclidean distance, i.e., $\mathcal{D}_{\text{spatial}}(\mathbf{t}_i, \mathbf{t}_j) \;=\; \|\mathbf{t}_i -\mathbf{t}_j\|_2^2$ between the 3D coordinates as the FPS metric.

\subsection{Coverage Objective}
Stage 2 always returns $B$ tokens for any feasible budget $B\le M$; this budget guarantee does not depend on the voxel
search. Coverage admits a bound in the common case where the safeguard is inactive, i.e. $|\mathcal{C}_{\mathrm{init}}|=G(v_s;X)\le B$: every unselected
token then shares a voxel of side $v_s$ with a selected one and lies within its space diagonal, yielding
$d_H(X,\mathcal{C}_{\mathrm{init}}) \le \sqrt{3}\,v_s$.
Each FPS step then selects the token attaining the inner $\max$--$\min$ of Eq.~\ref{eq:hausdorff}, so $d_H(X,\mathcal{C})$ is non-increasing through stage 2 (the squared Euclidean distance leaves the selection order unchanged) and the bound carries to the final selection. When the safeguard is active, the exact budget guarantee still holds, but this particular bound may no longer apply; full proofs are in the Appendix.

%% file: sec/4_experiment.tex
\section{Experiments and Results}

\textbf{Benchmarks and Metrics.} We evaluate all three forms of 3D reasoning: ScanQA~\cite{scanqa} for 3D spatial understanding, SQA3D~\cite{sqa} for situated reasoning grounded in an agent's position, and OpenEQA~\cite{openeqa} for open-vocabulary embodied QA. Together, they probe object recognition, attributes, counting, localization, and spatial relations. Following prior work~\cite{scanqa, openeqa, dtc, segpruner, 3dllm, llava3d, li2026geometry}, we report EM@1, CIDEr~\cite{cider}, and ROUGE-L~\cite{rouge} for ScanQA; EM@1 for SQA3D; and LLM-Match for OpenEQA. Efficiency is measured by inference and pruning time (s), LLM TFLOPs, KV-cache (MB), and peak GPU memory (GB).

\noindent\textbf{Models and protocol.} We test CoVeR with four VLMs: \texttt{LLaVA-OV-7B}~\cite{llava-onevision}, \texttt{Video-3D LLM}~\citep{video3dllm}, \texttt{Qwen2.5-VL-7B}~\cite{qwen2.5}, and \texttt{Qwen3-VL-8B}~\citep{qwen3}. Following~\cite{dtc, segpruner}, we uniformly sample $12$ views, evaluate prior work retention ratios, and compare with VTC~\cite{dtc}, DTC~\cite{dtc}, VisPruner~\cite{vispruner}, and SeGPruner~\cite{segpruner} under matched inputs. For Geo3DPruner~\cite{li2026geometry}, we follow its protocol with \texttt{Video-3D LLM} using 16 views. \footnote{Official code for VTC/ DTC and Geo3DPruner are not publicly available; we follow their reported protocol for a direct comparison.} CoVeR uses $\alpha = 0.4$ throughout; all experiments run on one NVIDIA H100 GPU. Additional details are in the Appendix.

\subsection{Main Results}
CoVeR achieves the best aggregate performance at every token budget (Table~\ref{tab:main_comparison}). Averaging across three datasets, CoVeR retains 93.5\% of full-token performance at 8\% token retention, versus 89.6\% and 85.9\% for SeGPruner and VisPruner. On ScanQA, it improves over the full-token baseline at 23\% retention, reaching 28.5 EM@1 and 85.5 CIDEr, while at 9\% budget it achieves 27.1 EM@1, 81.4 CIDEr, and 41.5 ROUGE-L, substantially outperforming SeGPruner and VisPruner. It also reaches 48.6 EM@1 on SQA3D at 8\% retention and outperforms prior pruning methods on OpenEQA at aggressive budgets. We further report category-level OpenEQA results and comparisons in the Appendix. Fig.~\ref{fig:qual_comp} shows CoVeR spreading tokens across all chairs and answering correctly, while VisPruner and SeGPruner cluster on a few patches and undercount.

\begin{table}[h]
\centering
\caption{\textbf{Performance Comparison}. CoVeR compared to~\textcolor{ExpCText}{\textit{Voxelization}} and \textcolor{ExpDText}{\textit{Learned Importance}} methods on 12-view ScanQA~\cite{scanqa}, OpenEQA~\cite{openeqa}, and SQA3D~\cite{sqa}. \textbf{Avg.} is over benchmarks, while \textbf{Rel.} is avg. \% of performance maintained. Higher is better.
}
\label{tab:main_comparison} 
\renewcommand{\arraystretch}{0.95}
\setlength{\tabcolsep}{1.25pt}
\resizebox{1.02\linewidth}{!}{
\begin{tabular}{lccccclr}
\toprule
\multirow{2}{*}{Methods} & \multicolumn{3}{c}{ScanQA} & OpenEQA & SQA3D & \multirow{2}{*}{Avg.} & \multirow{2}{*}{Rel.} \\
\cmidrule(lr){2-4}
\cmidrule(lr){5-5}
\cmidrule(lr){6-6}
& {\small EM@1} & {\small CIDEr} & {\small ROUGE-L} & {\small L-Match} & {\small EM@1} & & \\
\midrule

\multicolumn{8}{c}{Retain 100\% Tokens} \\
\midrule
\texttt{LLaVA-OV-7B} & 28.2 & 83.6 & 42.6 & 59.1 & 51.7 & 54.1 & 100.0 \\

\midrule
& \multicolumn{3}{c}{Retain 54\% Tokens}
& \multicolumn{2}{c}{Retain 56\% Tokens}
& & \\
\midrule
\textcolor{ExpCText}{DTC}~\cite{dtc} & 27.8 & -- & -- & -- & -- & -- & -- \\
\textcolor{ExpDText}{VisPruner}~\cite{vispruner} & 27.7 & 80.5 & 41.3 & 59.1 & 50.9 & 53.3 & 98.5 \\
\textcolor{ExpDText}{SeGPruner}~\cite{segpruner} & 28.5 & 83.8 & 42.6 & \cellcolor{tablecolor}\textbf{58.9} & \cellcolor{tablecolor}\textbf{51.7} & 54.1 & 100.0 \\
\textbf{CoVeR (Ours)} & \cellcolor{tablecolor}\textbf{28.7} & \cellcolor{tablecolor}\textbf{85.0} & \cellcolor{tablecolor}\textbf{43.2} & \cellcolor{tablecolor}\textbf{58.9} & \cellcolor{tablecolor}\textbf{51.7} & \cellcolor{tablecolor}\textbf{54.3} & \cellcolor{tablecolor}\textbf{100.4} \\

\midrule
& \multicolumn{3}{c}{Retain 40\% Tokens}
& \multicolumn{2}{c}{Retain 43\% Tokens}
& & \\
\midrule
\textcolor{ExpCText}{DTC}~\cite{dtc} & 27.7 & -- & -- & -- & -- & -- & -- \\
\textcolor{ExpDText}{VisPruner}~\cite{vispruner} & 28.0 & 80.3 & 41.4 & 58.4 & 51.0 & 53.1 & 98.3 \\
\textcolor{ExpDText}{SeGPruner}~\cite{segpruner} & 28.2 & 81.9 & 42.0 & 58.0 & 51.5 & 53.4 & 98.9 \\
\textbf{CoVeR (Ours)} & \cellcolor{tablecolor}\textbf{28.9} & \cellcolor{tablecolor}\textbf{85.5} & \cellcolor{tablecolor}\textbf{43.4} & \cellcolor{tablecolor}\textbf{58.6} & \cellcolor{tablecolor}\textbf{51.7} & \cellcolor{tablecolor}\textbf{54.3} & \cellcolor{tablecolor}\textbf{100.5} \\

\midrule
& \multicolumn{3}{c}{Retain 23\% Tokens}
& \multicolumn{2}{c}{Retain 26\% Tokens}
& & \\
\midrule
\textcolor{ExpCText}{DTC}~\cite{dtc} & 27.7 & -- & -- & -- & -- & -- & -- \\
\textcolor{ExpDText}{VisPruner}~\cite{vispruner} & 26.9 & 77.1 & 40.1 & 57.1 & 49.5 & 51.5 & 95.4 \\
\textcolor{ExpDText}{SeGPruner}~\cite{segpruner} & 27.7 & 78.9 & 40.7 & 57.5 & 50.6 & 52.4 & 97.1 \\
\textbf{CoVeR (Ours)} & \cellcolor{tablecolor}\textbf{28.5} & \cellcolor{tablecolor}\textbf{85.5} & \cellcolor{tablecolor}\textbf{43.2} & \cellcolor{tablecolor}\textbf{57.7} & \cellcolor{tablecolor}\textbf{51.7} & \cellcolor{tablecolor}\textbf{53.9} & \cellcolor{tablecolor}\textbf{99.7} \\

\midrule
& \multicolumn{3}{c}{Retain 14\% Tokens}
& \multicolumn{2}{c}{Retain 17\% Tokens}
& & \\
\midrule
\textcolor{ExpCText}{DTC}~\cite{dtc} & 26.7 & -- & -- & -- & -- & -- & -- \\
\textcolor{ExpDText}{VisPruner}~\cite{vispruner} & 24.8 & 71.7 & 37.5 & 55.9 & 49.0 & 49.9 & 92.2 \\
\textcolor{ExpDText}{SeGPruner}~\cite{segpruner} & 26.4 & 75.2 & 38.9 & 56.0 & 49.7 & 50.8 & 94.2 \\
\textbf{CoVeR (Ours)} & \cellcolor{tablecolor}\textbf{27.9} & \cellcolor{tablecolor}\textbf{82.4} & \cellcolor{tablecolor}\textbf{42.2} & \cellcolor{tablecolor}\textbf{56.8} & \cellcolor{tablecolor}\textbf{51.2} & \cellcolor{tablecolor}\textbf{52.9} & \cellcolor{tablecolor}\textbf{98.0} \\

\midrule
& \multicolumn{3}{c}{Retain 9\% Tokens}
& \multicolumn{2}{c}{Retain 8\% Tokens}
& & \\
\midrule
\textcolor{ExpCText}{DTC}~\cite{dtc} & 26.1 & -- & -- & -- & 48.0 & -- & -- \\
\textcolor{ExpDText}{VisPruner}~\cite{vispruner} & 23.4 & 66.9 & 35.6 & 51.5 & 45.7 & 46.4 & 85.9 \\
\textcolor{ExpDText}{SeGPruner}~\cite{segpruner} & 24.5 & 71.2 & 37.0 & 52.5 & 48.4 & 48.4 & 89.6 \\
\textbf{CoVeR (Ours)} & \cellcolor{tablecolor}\textbf{27.1} & \cellcolor{tablecolor}\textbf{81.4} & \cellcolor{tablecolor}\textbf{41.5} & \cellcolor{tablecolor}\textbf{53.0} & \cellcolor{tablecolor}\textbf{48.6} & \cellcolor{tablecolor}\textbf{50.5} & \cellcolor{tablecolor}\textbf{93.5} \\
\bottomrule
\end{tabular}
}
\end{table}

\begin{figure}[h]
\centering
\includegraphics[width=1.0\linewidth]{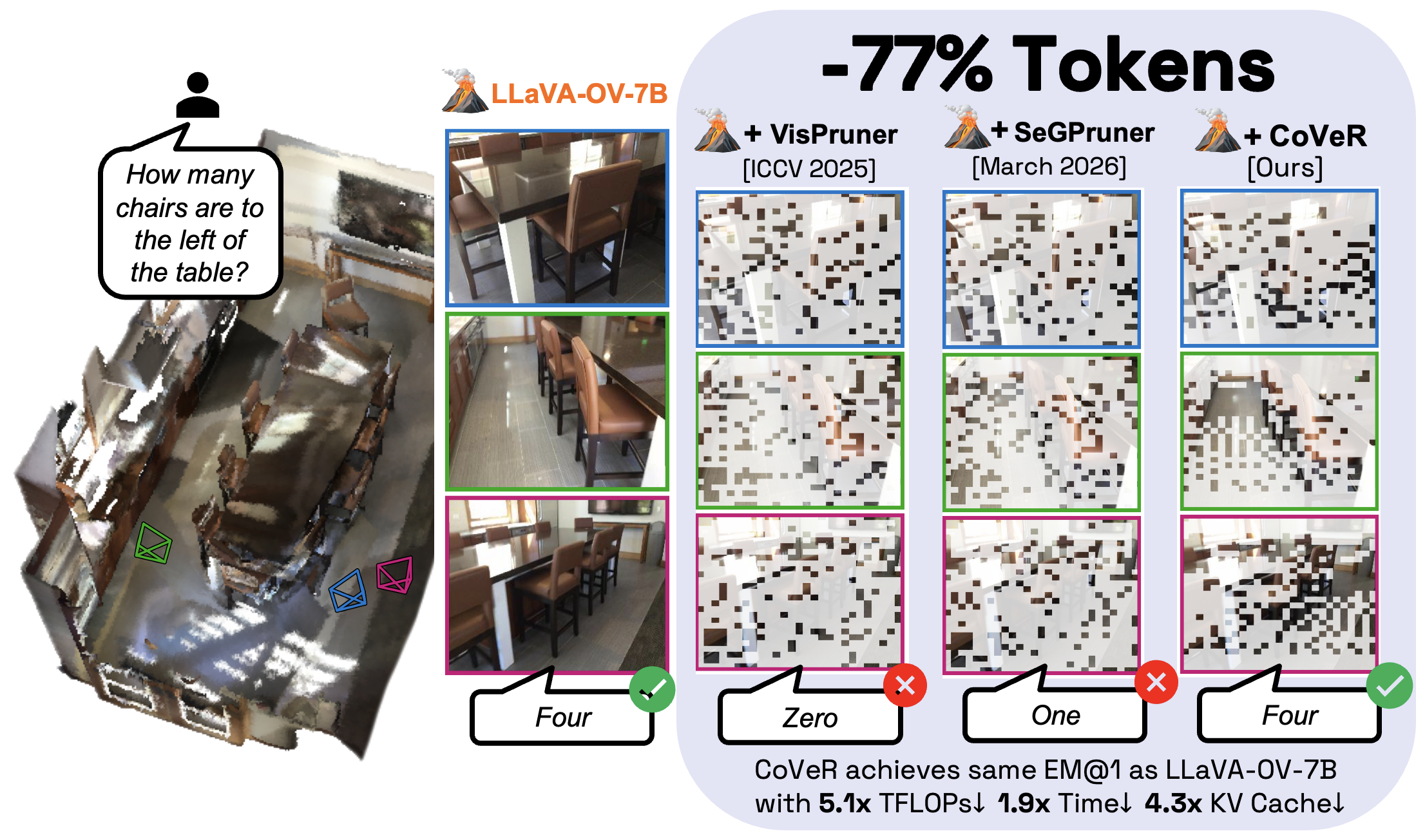}
\caption{\textbf{Qualitative comparisons.} CoVeR provides correct answers while substantially reducing the computational cost.}
\vspace{-2mm}
\label{fig:qual_comp}
\end{figure}

\subsection{Geometric Coverage Analysis}
\noindent \textbf{Does CoVeR improve Geometric Coverage?} Accuracy alone does not reveal how a pruning method should spend its target budget. We therefore measure coverage at the most aggressive ScanQA budget with two metrics. The \textit{Nearest Neighbor Index} (NNI)~\cite{nni} measures how evenly the retained tokens are spread: it is the ratio of their mean nearest-neighbor distance to the value expected under a spatially random process, so a low value flags the clustering we want to avoid. A spread-out set can still miss whole regions, so the \textit{Nearest Neighbor Distance} (NND$_q$) measures coverage of the scene directly: one minus the $q$-th percentile of every original token's distance to its nearest selected token, normalized by the scene diagonal. We report NND$_{95}$ for robustness and NND$_{100}$, the worst case, which equals the normalized complement of the directed Hausdorff distance minimized by CoVeR (Eq.~\ref{eq:hausdorff}). Higher indicates better coverage. Details are included in the Appendix.

As shown in Table~\ref{tab:spatial_main}, SeGPruner concentrates tokens on a few regions ($\mathrm{NNI}=0.458$), while CoVeR is far more uniform ($\mathrm{NNI}=0.924$). CoVeR also covers the scene better on both $\mathrm{NND}_{95}$ ($0.980$ vs.\ $0.967$) and $\mathrm{NND}_{100}$ ($0.977$ vs.\ $0.917$), with the largest gap on the worst case. All gaps are significant under the paired $t$- and Wilcoxon signed-rank tests. CoVeR also attains higher accuracy than SeGPruner, supporting our motivation that, at a fixed budget, spatial coverage beats concentrating tokens on a few regions. The widest gap falls on $\mathrm{NND}_{100}$, the normalized complement of the directed Hausdorff distance CoVeR minimizes, confirming it as the right objective for coverage-based selection.

\begin{table}[h]
\centering
\setlength{\tabcolsep}{3pt}
\caption{\textbf{Coverage and Performance}. CoVeR achieves substantially better spatial coverage with higher downstream 3D scene understanding performance compared to SeGPruner~\cite{segpruner}. Higher~$\uparrow$ is better on all columns.}
\label{tab:spatial_main}
\resizebox{\linewidth}{!}{
\begin{tabular}{l|ccc|ccc}
\toprule
\multirow{2}{*}{Method} & \multicolumn{3}{c|}{Coverage} & \multicolumn{3}{c}{Performance} \\
& NNI & NND$_{95}$ & NND$_{100}$ & EM@1 & CIDEr & ROUGE-L \\
\midrule
\textcolor{ExpDText}{SeGPruner} & 0.458 & 0.967 & 0.917 & 24.5 & 71.2 & 37.0 \\
\textbf{CoVeR} & \cellcolor{tablecolor}\textbf{0.924} & \cellcolor{tablecolor}\textbf{0.980} & \cellcolor{tablecolor}\textbf{0.977} & \cellcolor{tablecolor}\textbf{27.1} & \cellcolor{tablecolor}\textbf{81.4} & \cellcolor{tablecolor}\textbf{41.5} \\
\bottomrule
\end{tabular}}
\end{table}

\noindent \textbf{Does CoVeR preserve informative regions?} A geometry-only selector raises a concern: broad coverage might come from visually uninformative regions while discarding answer-critical evidence. We therefore compare CoVeR and SeGPruner selections directly at 9\% ScanQA retention with Token Recovery (TR) and Token Expansion (TE). TR measures the normalized distance from each SeGPruner token to its nearest CoVeR token, while TE is the reverse direction; both are defined in the supplementary materials. A small TR means CoVeR keeps a token near every region SeGPruner selects, while a large TE means CoVeR also covers regions SeGPruner ignores.
\vspace{0.25em}

\begin{figure}
  \centering
  \includegraphics[width=0.79\linewidth]{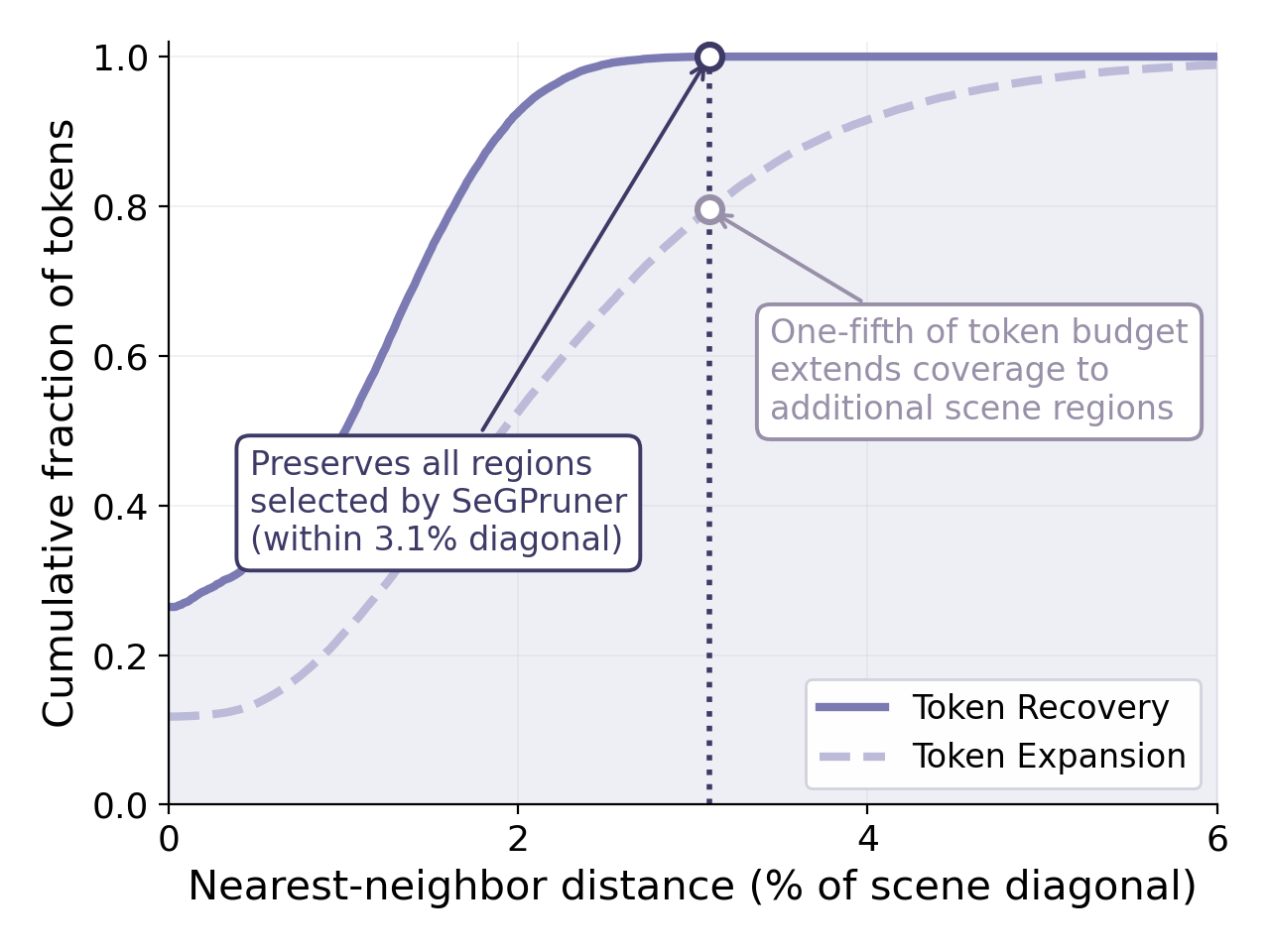}
  \caption{\textbf{Cumulative distributions of directed nearest-neighbors}. CoVeR preserves SeGPruner's selected regions within $3.1\%$ of scene diagonal while covering $\approx$20\% additional regions.}
  \label{fig:recovery_expansion}
\end{figure}

\begin{table}[h]
\centering
\setlength{\tabcolsep}{2pt}
\caption{\textbf{Directed distance}. CoVeR remains close to regions selected by SeGPruner while expanding into additional regions.}
\label{tab:spatial_rel}
\resizebox{0.60\linewidth}{!}{
\begin{tabular}{l|cc}
\toprule
\textbf{Comparison} & TR $\downarrow$ & TE $\uparrow$ \\
\midrule
\textbf{CoVeR} vs. \textcolor{ExpDText}{SeGPruner\ } & \cellcolor{tablecolor}\textbf{\ 0.009\ } & \cellcolor{tablecolor}\textbf{\ 0.020\ } \\
\bottomrule
\end{tabular}}
\vspace{-1em}
\end{table}

\noindent Table~\ref{tab:spatial_rel} shows TR is small ($0.009$ of the scene diagonal) and TE about twice as large ($0.020$), so CoVeR's tokens stay close to SeGPruner's while the reverse does not hold. The cumulative distributions (Fig.~\ref{fig:recovery_expansion}) make this sharper: every SeGPruner token lies within $3.1\%$ of the scene diagonal of a CoVeR token, while roughly $20\%$ of CoVeR tokens remain farther than that from any SeGPruner token. Thus, coverage preserves salient regions while additionally covering the rest of the scene.

\begin{callout}{\textbf{Insight 1:}}
{Broader spatial coverage accompanies stronger 3D reasoning. Prioritizing coverage does not exclude informative regions but instead preserves them, extending coverage to additional parts of the scene.}
\end{callout}

\subsection{Efficiency and Generalization}
\noindent \textbf{Efficiency.} Relative to \texttt{LLaVA-OV-7B} (Table~\ref{tab:efficiency_tradeoff}), CoVeR delivers substantial efficiency gains while largely preserving performance. At the most aggressive 9\% budget, it achieves $13.3\times$ fewer TFLOPs, $10.7\times$ smaller KV cache, and a $2.9\times$ speedup for a 1.1-point drop in performance. Against the \textcolor{ExpDText}{\textit{learned importance}} pruners, CoVeR also pairs the highest performance with the lowest peak GPU memory (Table~\ref{tab:efficiency_comparision}): they must run the encoder with attention outputs enabled and hold the attention maps to rank tokens, whereas CoVeR selects from 3D coordinates alone, with only the negligible overhead of a binary search plus FPS.

\noindent\textbf{Additional Backbones.} Under Geo3DPruner's~\cite{li2026geometry} 16-view, 10\% retention protocol on \texttt{Video-3D LLM}~\cite{video3dllm}, CoVeR obtains 26.5 ScanQA EM@1 versus 26.0 for Geo3DPruner, and retains 93.5\% of full performance versus 90.7\% (Table~\ref{tab:video3dllm}), even though Geo3DPruner adds a \texttt{VGGT-1B} encoder~\cite{vggt} and fully retrains the backbone. Without changing the selection rule or $\alpha$, CoVeR also transfers to \texttt{Qwen2.5-VL-7B} and \texttt{Qwen3-VL-8B}, which differ substantially in visual encoders, tokenization, and resolution handling. Both retain over 96\% of ScanQA performance at retention levels above 20\% (Fig.~\ref{fig:qwen_scores}(a)) and over 95\% on SQA3D until retention falls below 20\% (Fig.~\ref{fig:qwen_scores}(b)), matching trends across architectures.

\begin{table}[h]
\renewcommand{\arraystretch}{1.2}
\setlength{\tabcolsep}{2pt}
\centering
\caption{\textbf{Efficiency at varying token retention.} `Pruning' denotes the average time which CoVeR needs to select tokens, while `Time' reports end-to-end inference latency (in sec). CoVeR substantially reduces TFLOPs, KV cache (MB), memory (GB) while remaining competitive across token budgets. Results are relative to \texttt{LLaVA-OV-7B} on ScanQA.}
\label{tab:efficiency_tradeoff}
\resizebox{\linewidth}{!}{
\begin{tabular}{lccccccc}
\toprule
\multirow{2.5}{*}{\shortstack[l]{Tokens\\Retained}}
& \multicolumn{5}{c}{Efficiency}
& \multicolumn{2}{c}{Accuracy} \\
\cmidrule(lr){2-6}\cmidrule(lr){7-8}
& Pruning\,$\downarrow$
& Time\,$\downarrow$
& TFLOPs\,$\downarrow$
& KV\,$\downarrow$
& Mem\,$\downarrow$
& EM@1\,$\uparrow$
& \boldmath$\Delta$ \\
\midrule
100\% & -- & 0.497 & 145.5 & 480.0 & 24.1 & 28.2 & -- \\
\addlinespace[1pt]
54\% & 0.189 & 0.493\imp{1.0} & 71.1\imp{2.0}  & 259.9\imp{1.8} & 20.3\imp{1.2} & 28.7 & \dgood{+0.5} \\
40\% & 0.141 & 0.388\imp{1.3} & 51.1\imp{2.8}  & 192.9\imp{2.5} & 19.2\imp{1.3} & 28.9 & \dgood{+0.7} \\
23\% & 0.082 & 0.268\imp{1.9} & 28.3\imp{5.1}  & 111.7\imp{4.3} & 17.8\imp{1.4} & 28.5 & \dgood{+0.3} \\
14\% & 0.049 & 0.202\imp{2.5} & 17.0\imp{8.6}  & 68.6\imp{7.0}  & \cellcolor{tablecolor}\textbf{17.2}\imp{1.4} & 27.9 & \dbad{-0.3} \\
9\%  & \cellcolor{tablecolor}\textbf{0.034} & \cellcolor{tablecolor}\textbf{0.174}\imp{2.9} & \cellcolor{tablecolor}\textbf{10.9}\imp{13.3} & \cellcolor{tablecolor}\textbf{44.7}\imp{10.7} & \cellcolor{tablecolor}\textbf{17.2}\imp{1.4} & 27.1 & \dbad{-1.1} \\
\bottomrule
\end{tabular}}
\vspace{-0.5em}
\end{table}

\begin{table}[h]
    \vspace{-0.5em}
    \renewcommand{\arraystretch}{1.2}
    \setlength{\tabcolsep}{2pt}
    \centering
    \caption{\textbf{Efficiency comparison.} CoVeR attains the highest accuracy at the lowest peak memory. Its selection cost remains a fraction of end-to-end inference. Scores report pruning on \texttt{LLaVA-OV-7B} at 9\% token retention for ScanQA.}
    \resizebox{\linewidth}{!}{
    \begin{tabular}{lccccc}
        \toprule
        Methods & EM@1\,$\uparrow$ & CIDEr\,$\uparrow$  & ROUGE-L\,$\uparrow$ & Pruning (s)\,$\downarrow$ & Mem (GB)\,$\downarrow$ \\
        \midrule
        \textcolor{ExpDText}{VisPruner}~\cite{vispruner} & 23.4 & 66.9 & 35.6 & 0.010 & 22.1\\ 
        \textcolor{ExpDText}{SeGPruner}~\cite{segpruner} & 24.5 & 71.2 & 37.0 & \cellcolor{tablecolor}\textbf{0.008} & 22.1\\
        \textbf{CoVeR} & \cellcolor{tablecolor}\textbf{27.1} & \cellcolor{tablecolor}\textbf{81.4} & \cellcolor{tablecolor}\textbf{41.5} & 0.034 & \cellcolor{tablecolor}\textbf{17.2}\\
        \bottomrule
    \end{tabular}}
    \vspace{-0.75em}
    \label{tab:efficiency_comparision}
\end{table}

\begin{table}[h]
\centering
\caption{\textbf{Generalization with \texttt{Video-3D LLM}~\cite{video3dllm} backbone}. Geo3DPruner~\citep{li2026geometry} introduces a \texttt{VGGT-1B}~\citep{vggt} encoder and requires full backbone retraining, while CoVeR is training-free. Scores are EM@1\,$\uparrow$ on 16-view at 10\% budget.}
\renewcommand{\arraystretch}{1.15}
\setlength{\tabcolsep}{3pt}
\label{tab:video3dllm}
\resizebox{\linewidth}{!}{
\begin{tabular}{llcccccc}
\toprule
\multirow{2}{*}{Methods} & \multirow{2}{*}{Encoder} & \multirow{2}{*}{Retrain} & \multicolumn{2}{c}{ScanQA} & \multicolumn{2}{c}{SQA3D}  & \multirow{2}{*}{Rel.(\%)} \\
\cmidrule(lr){4-5} \cmidrule(lr){6-7}
& & & 100\% & 10\% & 100\% & 10\% & \\
\midrule
\textcolor{ExpDText}{Geo3DPruner}~\cite{li2026geometry} & \texttt{VGGT1B} & Full & 29.7 & 26.0 & 59.3 & 55.7 &  90.7 \\
\textbf{CoVeR} & None & None & 28.9 & 26.5 & 57.9 & 55.1 & \cellcolor{tablecolor}\textbf{93.5} \\
\bottomrule
\end{tabular}}
\vspace{-0.5em}
\end{table}

\begin{figure}[h]
\begin{center}
\includegraphics[width=\linewidth]{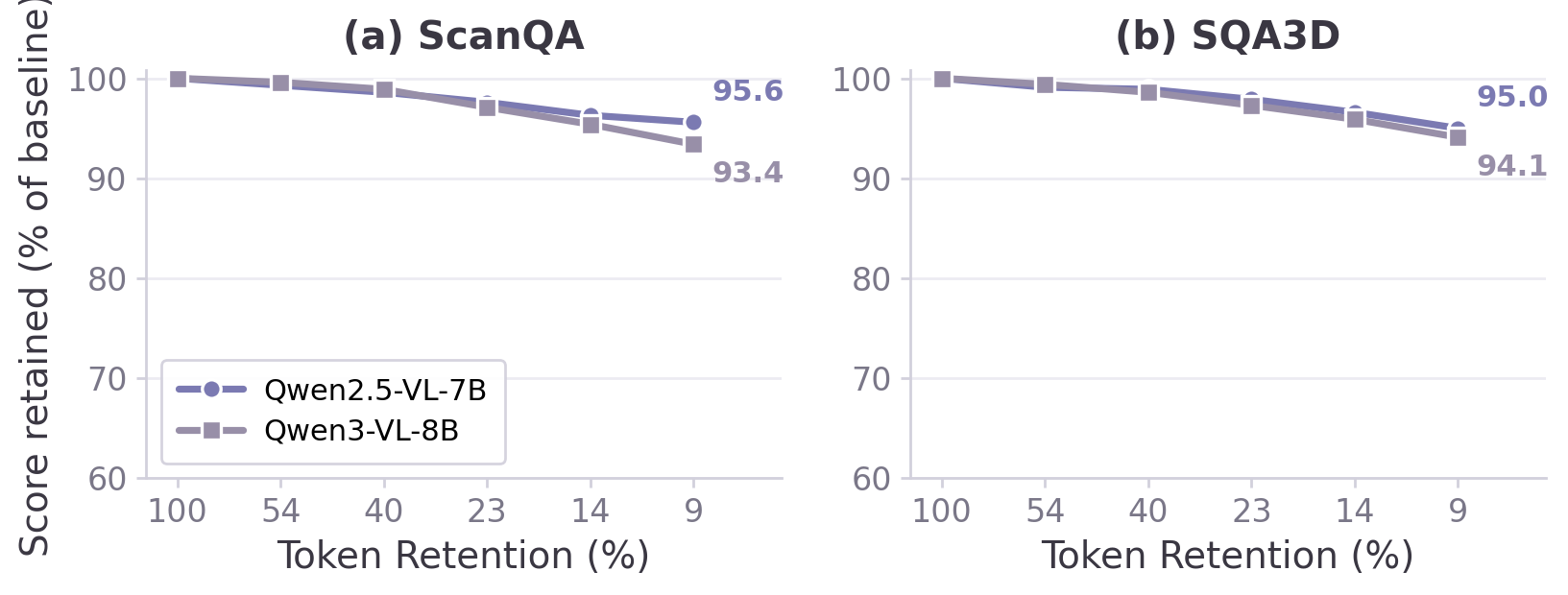}
\caption{\textbf{CoVeR transfers across VLMs.} CoVeR exhibits similar performance on \texttt{Qwen2.5-VL-7B} and \texttt{Qwen3-VL-8B}. Performance remains above 93\% of the baseline even at 9\% token retention for both (a) ScanQA and (b) SQA3D.}
\vspace{-2em}
\label{fig:qwen_scores}
\end{center}
\end{figure}

\subsection{Design Ablations}
We ablate CoVeR one component at a time. Full details have been provided in the Appendix.
\vspace{0.25em}

\noindent \textbf{Stage 1: Preserving encoder-native tokens.} Fig.~\ref{fig:voxel_ablation} shows that the advantage of keeping a real token (pruning) over averaging features within a voxel (merging) widens with stronger compression: at the tightest budget, pruning improves EM@1 by 7.7/18.1 on ScanQA/SQA3D, with the largest SQA3D gains on \texttt{Can} (+37) and \texttt{Which} (+29), which need spatial grounding. Coarse voxels mix objects, surfaces, and viewpoints, so averaging yields a synthetic feature, whereas CoVeR keeps an encoder-native token with a valid spatial identity. CoVeR also beats DTC~\cite{dtc} at every budget, so the gain is not merely voxel-size selection.

\begin{callout}{\textbf{Insight 2.}}
{Original encoder tokens are easier for LLM to interpret than synthetic avg. of heterogeneous regions.}
\end{callout}

\begin{figure}[h]
\begin{center}
\includegraphics[width=\linewidth]{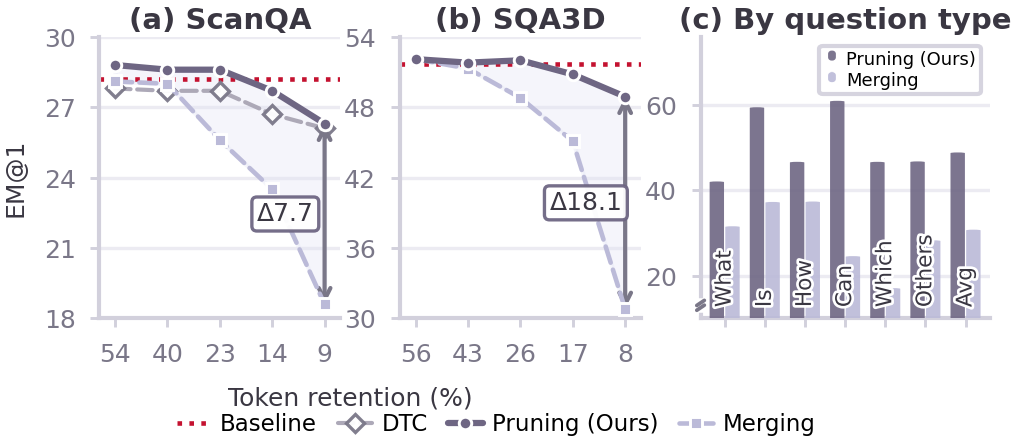}
    \vspace{-1em}
    \caption{\textbf{Coverage initialization analysis} ($\alpha$ = 1). Pruning consistently outperforms merging, with gap widening at tight budgets.}
    \vspace{-1em}
\label{fig:voxel_ablation}
\end{center}
\end{figure}

\noindent\textbf{Stage 2: Geometric distance.} Fig.~\ref{fig:fps_ablation} shows pure 3D distance is consistently strongest across budgets. At the tightest budget, it beats
spatial+semantic FPS by 1.7/2.1 on ScanQA/SQA3D and semantic-only FPS by 2.9/4.3, with the largest SQA3D gains on \texttt{How} (+6.8) and \texttt{What} (+6.2), which require counting or localizing multiple objects. Semantic FPS suppresses distinct objects with similar embeddings, while spatial distance keeps them when 3D locations differ. Without attention or features, it still matches or exceeds SeGPruner~\cite{segpruner}.
\begin{callout}{\textbf{Insight 3.}}
{Spatial distance preserves visually similar instances at different locations, while semantic distance can incorrectly suppress them.}
\end{callout}

\begin{figure}[h]
\begin{center}
\includegraphics[width=\linewidth]{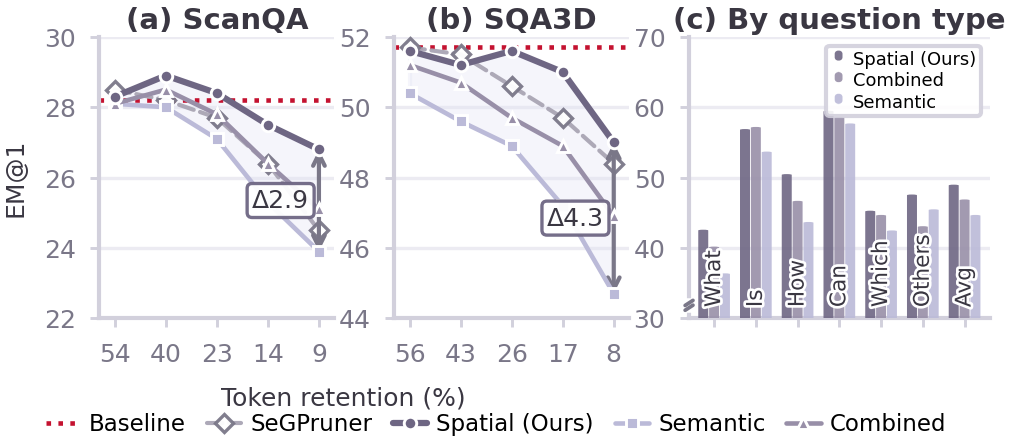}
\caption{\textbf{Coverage expansion analysis} ($\alpha$ = 0). Spatial distance outperforms combined and semantic alternatives, with gap widening at tight budgets.}
\vspace{-1em}
\label{fig:fps_ablation}
\end{center}
\end{figure}

\vspace{1em}
\noindent\textbf{CoVeR ablation.} Table~\ref{tab:ablation_combined}
ablates the expansion rule and its initialization. 
Our iterative FPS updates the minimum distance after each pick; we compare it against \textit{Top-$K$ selection} (selecting tokens farthest from the stage 1 tokens without distance updates) and \textit{Random selection} (uniform sampling). Both are weaker: at 9\% retention they lose 0.7 and 1.7 average ScanQA points, since only iterative updates keep placing tokens in still-uncovered regions. For initialization, seeding expansion from the stage 1 tokens rather than \textit{From scratch} adds 1.1, 0.2, and 0.5 points at 23\%, 14\%, and 9\% retention, because voxel seeding removes duplicates in parallel first, so expansion extends into new regions instead of rediscovering covered ones.

\noindent\textbf{Coverage initialization is necessary.} Table~\ref{tab:ablation_combined} isolates the contribution of stage~1 by comparing CoVeR with FPS only design. Without voxel initialization, FPS must build scene coverage sequentially from scratch.

\begin{table}[h]
\centering
\caption{\textbf{Design ablations.} Iterative FPS with stage 1 initialization consistently performs best. Results on ScanQA, averaged over EM@1, CIDEr, and ROUGE-L. Pruning time is reported in sec.}
\label{tab:ablation_combined}
{\resizebox{\linewidth}{!}{
\begin{tabular}{llccc}
\toprule
\multirow{2}{*}{Component} &
\multirow{2}{*}{Setting} &
\multicolumn{3}{c}{Token Budget} \\
\cmidrule(lr){3-5}
& & 23\% & 14\% & 9\% \\
\midrule

\multirow{3}{*}{Expansion}
& w/o Iterative FPS (Top-$K$)
& 51.8
& 50.4
& 49.3 \\
& w/o Iterative FPS (Random)
& 51.8
& 50.6
& 48.3 \\
& \textbf{w/ Iterative FPS}
& \cellcolor{tablecolor}\textbf{52.4}
& \cellcolor{tablecolor}\textbf{50.8}
& \cellcolor{tablecolor}\textbf{50.0} \\
\midrule

\multirow{2}{*}{Initialization}
& w/o Stage 1 seed (From scratch)
& 51.3
& 50.6
& 49.5 \\
& \textbf{w/ Stage 1 seed}
& \cellcolor{tablecolor}\textbf{52.4}
& \cellcolor{tablecolor}\textbf{50.8}
& \cellcolor{tablecolor}\textbf{50.0} \\
\midrule
\multirow{2}{*}{Pruning time}
& w/o Stage 1 (FPS only) & 0.126 & 0.078 & 0.049 \\
& \textbf{w/ Stage 1 (CoVeR)} & \cellcolor{tablecolor}\textbf{0.082} & \cellcolor{tablecolor}\textbf{0.049} & \cellcolor{tablecolor}\textbf{0.034} \\
\bottomrule
\end{tabular}}}
\end{table}

\noindent Seeding FPS with one representative per occupied voxel instead provides a coarse coverage from the start, which reduces pruning time by approximately $1.5\times$. Across all variants above, higher coverage (NNI, NND$_{95}$, NND$_{100}$) accompanies higher accuracy, with full CoVeR best on every measure (Fig.~\ref{fig:coverage_ablation}). Weaker variants leave coherent regions uncovered (top-$K$, random) or revisit initialized voxels (from-scratch FPS). This consistent ranking supports coverage as the mechanism behind the downstream gains.

\begin{figure}[h]
\centering
\includegraphics[width=\linewidth]{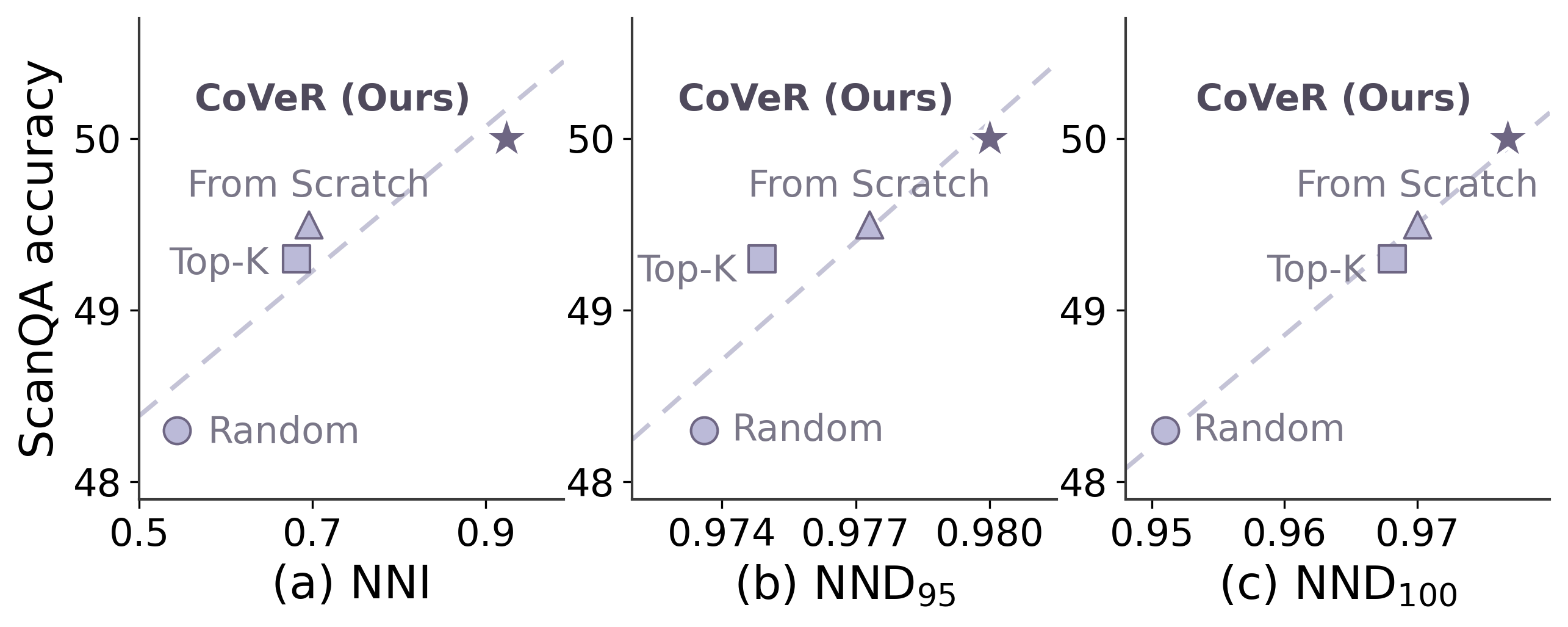}
\caption{\textbf{Better coverage achieves higher performance.} Correlation between coverage and accuracy (on ScanQA at 9\% retention).}
\vspace{-0.5em}
\label{fig:coverage_ablation}
\end{figure}

\noindent\textbf{Voxel ratio.} Varying $\alpha$ from 0.1 to 0.9 changes accuracy within a narrow band (Fig.~\ref{fig:alpha_study}), so the split between the two stages is not narrowly tuned; we use $\alpha=0.4$ throughout.

\noindent\textbf{Scaling with views.} CoVeR leads VisPruner and SeGPruner at every view count and gains more from added views (Fig.~\ref{fig:teasor} and Fig.~\ref{fig:view_scaling}), because coverage reallocates the fixed budget to newly visible regions rather than repeatedly observed regions favored by attention.

\begin{figure}[h]
    \centering
    \vspace{-0.5em}
    \begin{subfigure}[t]{0.49\linewidth}
        \centering
        \includegraphics[width=\linewidth]{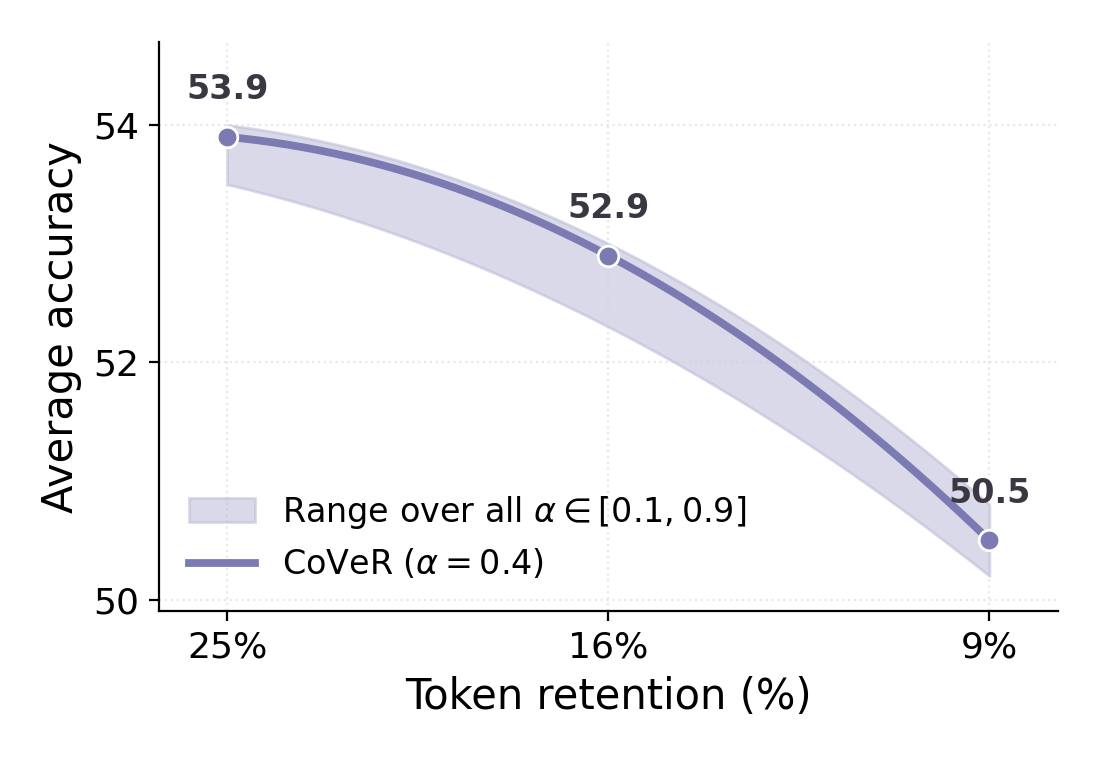}
        \caption{\textbf{Voxel ratio.} Avg. accuracy across ScanQA, SQA3D, OpenEQA remains stable across budgets.}
        \label{fig:alpha_study}
    \end{subfigure}
    \hfill
    \begin{subfigure}[t]{0.49\linewidth}
        \centering
        \includegraphics[width=\linewidth]{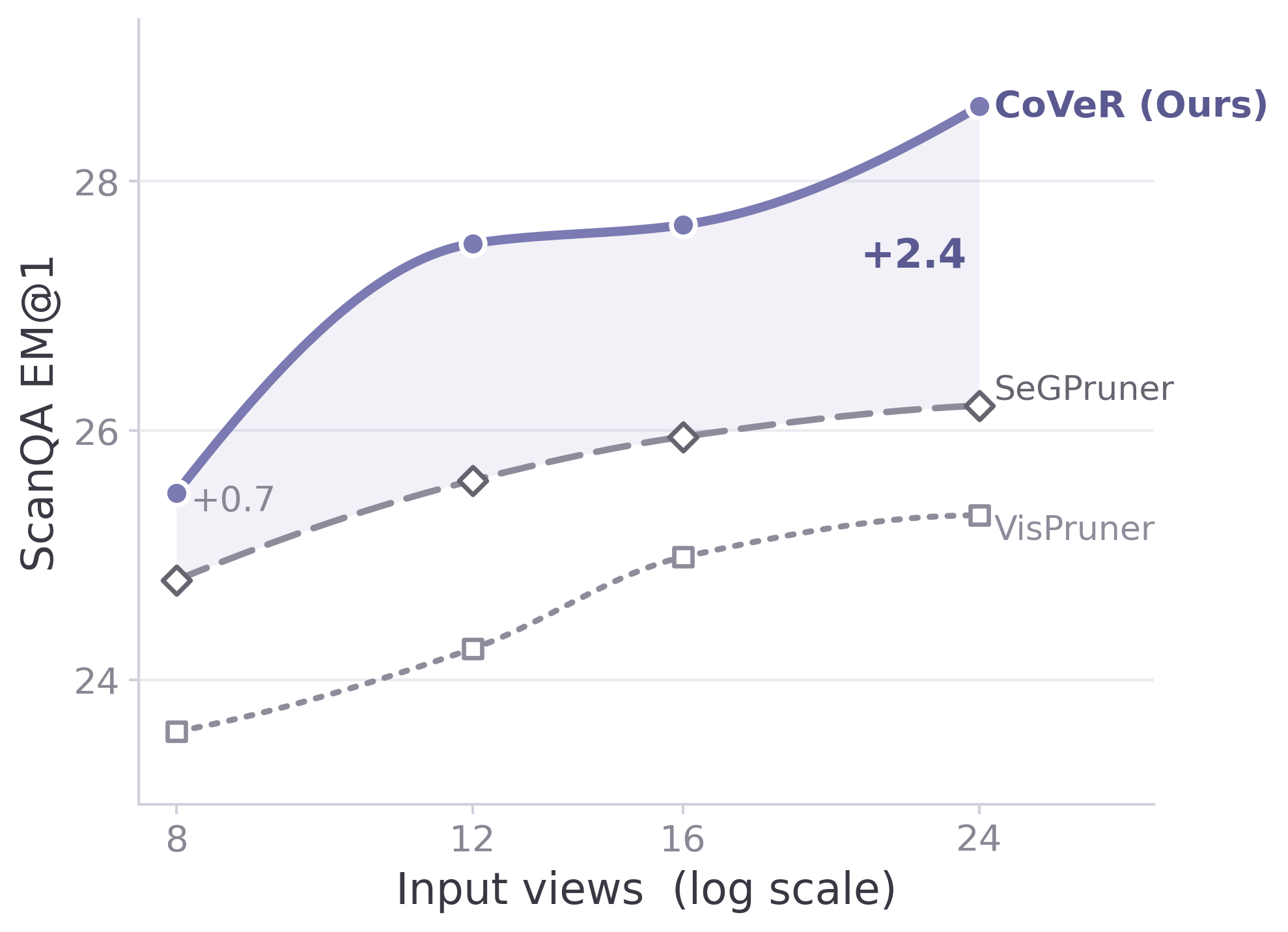}
        \caption{\textbf{Effect of View count.} CoVeR improves with more views, while VisPruner and SeGPruner saturate.}
        \label{fig:view_scaling}
    \end{subfigure}
    \caption{\textbf{Ablation for voxel ratio and scaling view count.}}
    \label{fig:additional_analysis}
    \vspace{-1em}
\end{figure}

%% file: sec/5_conclusion.tex
\section{Conclusion}

We propose CoVeR, a training-free, deterministic, geometry-only framework for reducing the visual token count of multi-view 2D VLMs. Our analysis identifies key limitations in each prior family: \textcolor{ExpCText}{\textit{voxelization}}-based methods cannot enforce exact per-scene budgets and are capped by voxel saturation, while \textcolor{ExpDText}{\textit{learned importance}}-based methods concentrate on prominent regions rather than scene coverage. CoVeR addresses both: it searches per scene for the voxel size that yields an initial selection, then extends beyond the saturation plateau using spatial distance alone. It achieves SOTA on three benchmarks while generalizing across four VLMs from two model families without retraining. 
\vspace{0.25em}

\noindent \textbf{Limitations.} Like all prior methods, CoVeR requires depth and camera, and is designed for indoor scenes. Its performance may therefore depend on the quality of the estimated geometry. Future work can combine coverage with reliable depth/pose estimation, and hierarchical or streaming selection for outdoor scenes.

\section*{Acknowledgments}

The authors thank Saswat Subhajyoti Mallick, Surgan Jandial, Nicholas Mesa-Cucalon, and Yinong Oliver Wang for their insightful discussions, feedback, and assistance with parts of the project. Aviral Chharia was supported in part by the Uber Presidential Fellowship from the Robotics Institute, Carnegie Mellon University. The computational resources were supported in part by PSC Bridges-2 through the Advanced Cyberinfrastructure Coordination Ecosystem: Services and Support (ACCESS) program allocation CIS250962, which is supported by National Science Foundation (NSF) grants \#2138259, \#2138286, \#2138307, \#2137603, and \#2138296.\\
\vspace{-1em}

%% file: sec/6_appendix.tex
\section{Implementation Details}
\label{sec:experiement_details}
\subsection{Benchmarks}
To demonstrate generalization across diverse 3D scene understanding tasks, we evaluate CoVeR on benchmarks covering spatial scene understanding (ScanQA~\cite{scanqa}), situated reasoning (SQA3D~\cite{sqa}), and embodied question answering (OpenEQA~\cite{openeqa}).

ScanQA~\cite{scanqa} and SQA3D~\cite{sqa} are both based on the scenes from the ScanNet~\cite{dai2017scannet} dataset. ScanQA requires models to answer questions related to spatial understanding in a 3D environment, while SQA3D requires models to have situated reasoning awareness, in which agents must determine their position and orientation in order to correctly answer the questions. Following previous works~\cite{dtc, segpruner}, we include results on the validation set of ScanQA and the test set of SQA3D. 

We additionally show results on the OpenEQA benchmark. OpenEQA~\cite{openeqa} is the first open-vocabulary benchmark for embodied question answering (EQA), built on ScanNet~\cite{dai2017scannet} and HM3D~\cite{ramakrishnan2021hm3d} datasets. The benchmark includes seven test aspects: object recognition, attribute recognition, object state recognition, object localization, spatial reasoning, functional reasoning, and world knowledge. Together, these datasets demonstrate that CoVeR generalizes to 3D spatial understanding. Additional details for these datasets are included in Table~\ref{tab:stat_datasets} and~\ref{tab:stat_openeqa}.

\begin{table}[h]
\centering
\caption{Number of questions and scenes in each benchmark.}
\label{tab:stat_datasets}
\resizebox{0.72\linewidth}{!}{\begin{tabular}{l|c|c}
\hline
\textbf{Benchmark} & \textbf{\# of questions} & \textbf{\# of scenes} \\
\hline
ScanQA~\cite{scanqa} & 4,306 & 71 \\
OpenEQA~\cite{openeqa} & 1,636 & 152 \\
SQA3D~\cite{sqa} & 3,519 & 67 \\
\hline
\end{tabular}}
\vspace{-1em}
\end{table}

\begin{table}[h]
\centering
\caption{Number of questions and scenes of ScanNet~\cite{dai2017scannet} and HM3D~\cite{ramakrishnan2021hm3d} subsets in the OpenEQA~\cite{openeqa} benchmark.}
\label{tab:stat_openeqa}
\resizebox{0.7\linewidth}{!}{\begin{tabular}{l|c|c}
\hline
\textbf{Subset} & \textbf{\# of questions} & \textbf{\# of scenes} \\
\hline
ScanNet~\cite{dai2017scannet} & 1,079 & 89 \\
HM3D~\cite{ramakrishnan2021hm3d} & 557 & 63 \\
\hline
\cellcolor{tablecolor} Total & \cellcolor{tablecolor} 1,636 & \cellcolor{tablecolor} 152 \\
\hline
\end{tabular}}
\end{table}

\subsection{Baselines}
To ensure a fair comparison, we reproduce the results using the official implementations of VisPruner~\cite{vispruner} and SeGPruner~\cite{segpruner} in the same environment, with the default importance ratio set to $0.5$ (i.e., half of the selected tokens are chosen by their attention-based importance scores, while the remaining tokens are selected by the second stage of each method).

\subsection{Evaluation Metrics}
\textbf{EM@1} (Exact Match at top-1) is a binary string matching metric. It gives a score of $1$ when the prediction matches the ground truth exactly, and $0$ otherwise.

\noindent\textbf{CIDEr}~\cite{cider} measures similarity between the prediction and multiple reference answers for the same question using Term Frequency Inverse Document Frequency (TF-IDF) weighted $n$-gram matching, where higher scores indicate better alignment with ground truth. By default, $n$-gram is set to $4$.

\noindent\textbf{ROUGE-L}~\cite{rouge} measures the overlap between the prediction and the ground truth using Longest Common Subsequence (LCS), rewarding answers that preserve the word order of the ground truth.

\noindent\textbf{LLM-Match}~\cite{openeqa}, is used to evaluate the open-ended answers from VLMs:
\begin{equation}
    \text{LLM-Match} = \frac{1}{N} \sum_{i}^{N} \frac{\sigma_i - 1}{4} \times 100\% 
\end{equation}
where, $\sigma_i \in \{1,\dots,5\}$ is the score assigned by an LLM (\texttt{GPT-4o}~\citep{gpt4o} or \texttt{GPT-4}~\citep{gpt4} \footnote{As \texttt{GPT-4} is being deprecated: \url{https://developers.openai.com/api/docs/deprecations}, we report both \texttt{GPT-4} (Table~\ref{tab:openeqa_gpt4}) and \texttt{GPT-4o} (Table~\ref{tab:main_comparison} and Table~\ref{tab:openeqa_gpt4o}).}). A score of $1$ indicates an irrelevant answer, while $5$ denotes a correct answer. We follow the official prompt template released by OpenEQA, as shown in Figure~\ref{fig:openeqa_prompt}.

\noindent\textbf{LLM TFLOPs} measures the prefill computational cost of the LLM decoder. For a prompt of length $n$ (including text tokens and retained visual tokens) processed by a $L$-layer transformer decoder, the prefill FLOPs are calculated as:
\begin{equation}
\text{FLOPs} = L\Big(\underbrace{4 n d^{2} + 4\, n\, d\, d_{kv}}_{\text{attention (GQA)}} + \underbrace{4 n^{2} d}_{\text{attention}} + \underbrace{6\, n\, d\, m}_{\text{SwiGLU FFN}} \Big)
\end{equation}
where $d$ is the hidden size, $m$ is the feed-forward network (FFN) dimension, and $d_{kv} = H_{kv}\,(d / H_{q})$ denotes the key/value width under grouped-query attention (GQA), with $H_{q}$ query heads and $H_{kv}$ key/value heads.

\begin{figure}[t]
    \centering
    \begin{tcolorbox}[title=Prompt Template for LLM-Match Scores, fontupper=\small, fonttitle=\footnotesize]
    You are an AI assistant who will help me to evaluate the response given the question, the correct answer, and extra answers that are also correct. To mark a response, you should output a single integer between 1 and 5 (including 1, 5). 5 means that the response perfectly matches the answer or any of the extra answers. 1 means that the response is completely different from the answer and all of the extra answers.
    \\
    
    Example 1:
    
    Question: Is it overcast?
    
    Answer: no
    
    Extra Answers: ['doesn't look like it', 'no',' it's sunny']
    
    Response: yes
    
    Your mark: 1
    \\
    
    Example 2:
    
    Question: Who is standing at the table?
    
    Answer: woman
    
    Extra Answers: ['a woman', 'a lady', 'woman']
        
    Response: Jessica
    
    Your mark: 3
    \\
    
    Example 3:
    
    Question: Are there drapes to the right of the bed?
    
    Answer: yes
    
    Extra Answers: ['yes, there are drapes', 'yeah', 'the drapes are to the right of the king bed']
    
    Response: yes
    
    Your mark: 5
    \\
    
    Your Turn:
    
    Question: \{question\}
    
    Answer: \{answer\}
    
    Extra Answers: \{extra\_answers\}

    Response: \{prediction\}
    \end{tcolorbox}
    \caption{Prompt used for computing LLM-Match scores on the OpenEQA~\cite{openeqa} benchmark.}
    \label{fig:openeqa_prompt}
\end{figure}

\subsection{Backbones}
\textbf{LLaVA-OneVision-7B}~\cite{llava-onevision} consists of a visual encoder (SigLIP~\cite{siglip}), a projection layer (MLP), and a language model (Qwen2~\cite{qwen2llm}). The number of tokens for one input image is $729$ with $384 \times 384$ resolution, resulting in $8{,}748$ visual tokens for a $12$-view setting.

\noindent\textbf{Qwen2.5-VL-7B}~\cite{qwen2.5} contains a native dynamic resolution ViT~\citep{vit} visual encoder, an MLP-based Vision-Language Merger projection layer, and a language model (Qwen2.5). Processing the inputs yields $391$ tokens per image, resulting in $4{,}692$ visual tokens for the $12$-view setting.

\noindent\textbf{Qwen3-VL-8B}~\cite{qwen3} contains a visual encoder (SigLIP2~\cite{siglip2}), a projection layer similar to Qwen2.5-VL~\cite{qwen2.5}, and a language model (Qwen3~\cite{qwen3llm}). Processing the inputs yields $300$ tokens per image, resulting in $3{,}600$ visual tokens for a $12$-view setting.

\noindent\textbf{Video-3D-LLM}~\citep{video3dllm} extends \texttt{LLaVA-Video}~\citep{zhang2025llavavideovideoinstructiontuning} with sinusoidal 3D positional encoding added to the visual tokens to inject 3D camera geometry. Each frame produces 196 tokens, resulting in $3{,}136$ visual tokens for the $16$-frame setting.

\subsection{Packages}
To ensure the reproducibility of our results, we provide the version of all packages in Table~\ref{tab:conda_env}.

\begin{table}[h]
\centering
\caption{Packages version in the Conda environment.}
\label{tab:conda_env}
\begin{tabular}{l|l}
\hline
\textbf{Name} & \textbf{Version} \\
\hline
\texttt{python} & \texttt{3.10.20} \\
\texttt{torch} & \texttt{2.6.0+cu124} \\
\texttt{torchvision} & \texttt{0.21.0+cu124} \\
\texttt{numpy} & \texttt{2.2.6} \\
\texttt{pillow} & \texttt{12.1.1} \\
\texttt{transformers} & \texttt{5.8.0.dev0} \\
\texttt{tokenizers} & \texttt{0.22.2} \\
\texttt{accelerate} & \texttt{1.13.0} \\
\texttt{safetensors} & \texttt{0.7.0} \\
\texttt{huggingface-hub} & \texttt{1.13.0} \\
\texttt{flash-attn} & \texttt{2.8.3} \\
\texttt{qwen-vl-utils} & \texttt{0.0.14} \\
\hline
\end{tabular}
\end{table}

\subsection{Coverage Metrics}
\noindent\textbf{Coverage metrics.} 
For a point $\mathbf{t}$ and a non-empty index
set $\mathcal{A}$, let $\delta(\mathbf{t};\mathcal{A})=\min_{j\in\mathcal{A}}\|\mathbf{t}-\mathbf{t}_j\|_2$ be the distance from $\mathbf{t}$ to its nearest token in $\mathcal{A}$. We quantify how well the selected set $\mathcal{C}$ covers the full-token set $X$ with two metrics.

\begin{enumerate} 
\item \textit{Nearest Neighbor Index (NNI)} is the ratio of the observed mean nearest-neighbor distance among selected tokens to that expected under complete spatial randomness. As tokens lie in 3D, we use a 3D homogeneous Poisson process as the baseline:
\begin{equation}
\begin{aligned}
     \mathrm{NNI} = \frac{r_A}{r_E}, \qquad r_A &= \frac{1}{N} \sum_{i=1}^{N} \min_{j \neq i} \|\mathbf{t}_i-\mathbf{t}_j\|_2, \\
    r_E &= \Gamma\!\left(\frac{4}{3}\right) \left( \frac{4\pi}{3}\lambda \right)^{-1/3} 
\end{aligned}
\end{equation}
where $r_A$ is the observed mean nearest neighbor distance among the selected tokens, $r_E$ is its expectation under the Poisson process, and $\lambda = N/V$ is the density of $N$ selected tokens in a scene of bounding-box volume $V$. $\Gamma (.)$ denotes the Gamma function.
\item \textit{Nearest Neighbor Distance (NND)} measures scene coverage directly, since a spread-out set can still miss whole regions. For each original token we take its distance to the nearest selected token, and summarize by a percentile, normalized by the scene diagonal:
\begin{equation}
\begin{split}
\mathrm{NND}_{q} &= 1 - \frac{Q_{q}\big(\{\delta(\mathbf{x};\mathcal{C})\}_{\mathbf{x}\in X}\big)}{\mathrm{diag}}, \\
&\text{where } 
\begin{cases}
  q = 95: & \text{95th percentile} \\
  q = 100: & \max_{\mathbf{x}\in X}\delta(\mathbf{x};\mathcal{C}) = d_H(X,\mathcal{C})
\end{cases}
\end{split}
\end{equation}
where $Q_q$ is the $q$-th percentile over $X$. We report $\mathrm{NND}_{95}$, which clips the worst $5\%$ for robustness, and $\mathrm{NND}_{100}$, the worst case, whose unnormalized numerator equals the directed Hausdorff distance $d_H(X,\mathcal{C})=\max_{\mathbf{t}\in X}\delta(\mathbf{t};\mathcal{C})$ that CoVeR minimizes. Higher is better for both.
\end{enumerate}
\noindent\textbf{Directed distances.} To verify that coverage gains do not drop the salient regions chosen by \textcolor{ExpDText}{\textit{learned importance}}, we compare two selections $\mathcal{C}$ and $\mathcal{S}$ in both directions, normalized by the scene diagonal.
\begin{enumerate}
    \item \textit{Token Recovery (TR)} checks whether $\mathcal{C}$ keeps the regions $\mathcal{S}$ selects: for each token in $\mathcal{S}$ we take $\delta$ to the nearest token in $\mathcal{C}$, then average. A small TR means $\mathcal{C}$ has a token near every region $\mathcal{S}$ selects.
\begin{equation}
\begin{aligned}
\mathrm{TR}&=\frac{1}{|\mathcal{S}|}\sum_{\mathbf{s}\in\mathcal{S}}\frac{\delta(\mathbf{s};\mathcal{C})}{\mathrm{diag}}, \\
\mathrm{diag}
&= \sqrt{\sum_{d \in \{x,y,z\}}
\left(c_d^{\max} - c_d^{\min}\right)^2}
\end{aligned}
\end{equation}
where `diag' is the diagonal of the scene's axis-aligned bounding box.
    \item \textit{Token Expansion (TE)}  checks whether $\mathcal{C}$ reaches regions $\mathcal{S}$ ignores: for each token in $\mathcal{C}$ we take $\delta$ to the nearest token in $\mathcal{S}$, then average. A large TE means $\mathcal{C}$ covers regions $\mathcal{S}$ leaves out.
\begin{equation}
\mathrm{TE}=\frac{1}{|\mathcal{C}|}\sum_{\mathbf{c}\in\mathcal{C}}\frac{\delta(\mathbf{c};\mathcal{S})}{\mathrm{diag}}
\end{equation}
\end{enumerate}

\subsection{Ablation Study}

\noindent \textbf{Stage 1: Preserving encoder-native tokens.} To isolate the representation used in coverage initialization, we set $\alpha=1$ so stage 1 alone targets the full budget $B_{\text{init}}=B$; the highest budget in this study (54\%) stays within the saturation range of all ScanNet scenes (Sec.~\ref{sec:voxel_limits}). For a fair comparison, we modify the search to ensure exactly $B$ tokens for both strategies: rather than allowing $G(v_s;X)<B$ near saturation, we take the smallest voxel size with $G(v_s;X)\ge B$ and, if more than $B$ voxels form, keep the $B$ densest. This is used only for this ablation. We define:
\begin{enumerate} 
\item \textit{Pruning}, which keeps the representative token of each voxel (Eq.~\ref{eq:medoid}), so the voxel's output feature is a real encoder feature:
\begin{equation} 
\mathbf{f}_{\mathcal{V}_k} = \mathbf{f}_{i_k^{\text{rep}}}
\label{eq:pruning}
\end{equation}
\item \textit{Merging}, following VTC~\citep{dtc}, which averages all token features in the voxel into a pooled representation:
\begin{equation}
\mathbf{f}_{\mathcal{V}_k}=\frac{1}{|\mathcal{V}_k|}\sum_{i\in\mathcal{V}_k}\mathbf{f}_i
\label{eq:merging}
\end{equation}
\end{enumerate}
Both yield one feature per occupied voxel on the same voxel partition; they differ only in whether that feature is a real token (\textit{pruning}) or a synthetic average (\textit{merging}).

\noindent\textbf{Stage 2: Geometric distance.} To isolate expansion from stage 1, we set $\alpha=0$, so the stage 2 targets the full budget $B_{\text{expan}}=B$. For tokens $i,j$ with coordinates $\mathbf{t}_i,\mathbf{t}_j$ and $\ell_2$-normalized features $\hat{\mathbf{f}}_i,\hat{\mathbf{f}}_j$, we define three distances:

\begin{enumerate} 
\item \textit{Spatial}: squared Euclidean distance between 3D coordinates, $\mathcal{D}_{\text{spatial}}=\|\mathbf{t}_i-\mathbf{t}_j\|_2^2$.
\item \textit{Semantic}: cosine distance between features,
$\mathcal{D}_{\text{semantic}}=1-\hat{\mathbf{f}}_i^{\top}\hat{\mathbf{f}}_j$.
\item \textit{Combined}, following SeGPruner~\citep{segpruner}, which fuses both. As the terms differ in scale, we propose the normalization for each term: the spatial term by the squared scene diagonal and the semantic term by the cosine range:
\begin{equation}
\mathcal{D}_{\text{combined}}= \frac{\mathcal{D}_{\text{spatial}}}{\mathrm{diag}^2}
+ \frac{\mathcal{D}_{\text{semantic}}}{2}
\label{eq:combined}
\end{equation}
where $c_d^{\max}$ and $c_d^{\min}$ are the scene's max/min coordinates along $d$.
\end{enumerate}

\begin{figure}[h]
  \centering
  \includegraphics[width=\linewidth]{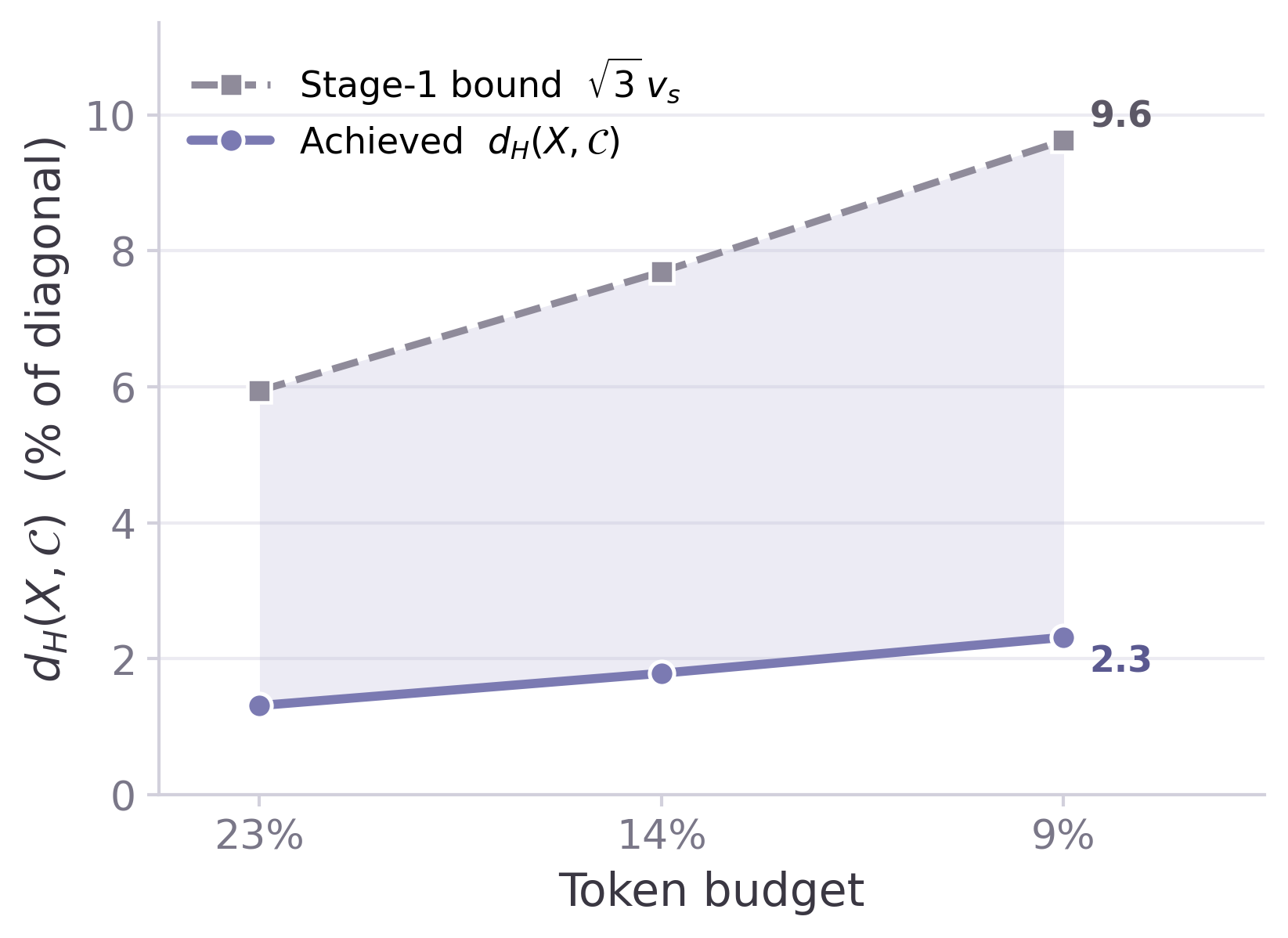}
  \caption{\textbf{Empirical validation of the coverage bound.} Even at 9\% token retention, the  $d_H(X,\mathcal{C})$ is only 2.3\% of the scene diagonal, compared with the stage 1 bound of 9.6\%.}
  \label{fig:appendix_b}
\end{figure}

\section{Theoretical Analysis of CoVeR}
\label{app:theoerical_proof}

\noindent\textbf{Setup.} Let $X=\{\mathbf{t}_i\}_{i=1}^{M}$ be the valid back-projected coordinates (Sec.~\ref{sec:problem_formulation}), and $\mathcal{A}\subseteq\{1,\dots,M\}$ a selection. Recall $\delta(\mathbf{t};\mathcal{A})$ and $d_H(X,\mathcal{A})$ above. We have $\mathcal{C}_{\text{init}}$ for the stage 1 selection at voxel size $v_s$, $m_1^*,\dots,m_{B_{\text{expan}}}^*$ for the stage 2 picks by FPS in order, $\mathcal{C}^j=\mathcal{C}_{\text{init}}\cup\{m_1^*,\dots,m_j^*\}$ (so $\mathcal{C}^0=\mathcal{C}_{\text{init}}$), and $\mathcal{C}=\mathcal{C}^{B_{\text{expan}}}$ for the final selection. We assume $B<M$ and $|\mathcal{C}_{\text{init}}|\le B$
(cases 1--2 of Sec.~\ref{sec:overview}), while case 3 is Remark~2.
\begin{lemma}[Stage 1 bound]
$d_H(X,\mathcal{C}_{\text{init}})\le\sqrt{3}\,v_s$.
\end{lemma}
\begin{proof}
Each $\mathbf{t}_i$ lies in one occupied voxel $\mathcal{V}_k$, and
stage 1 keeps one token $i_k^{\text{rep}}\in\mathcal{V}_k$. Both $\mathbf{t}_i$ and $\mathbf{t}_{i_k^{\text{rep}}}$ lie in a cube of side $v_s$, whose greatest internal distance is $\sqrt{3}\,v_s$, hence $\delta(\mathbf{t}_i;\mathcal{C}_{\text{init}})\le\sqrt{3}\,v_s$ for all $i$, and the maximum over $X$ gives the claim. 
\end{proof}

\begin{lemma}[Stage 2 bound]
$d_H(X,\mathcal{C}^{j+1})\le d_H(X,\mathcal{C}^j)$, hence
$d_H(X,\mathcal{C})\le d_H(X,\mathcal{C}_{\text{init}})$.
\end{lemma}
\begin{proof}
Since $\mathcal{C}^{j+1}\supseteq\mathcal{C}^j$, the minimum in
$\delta$ is over a superset, so $\delta(\mathbf{t};\mathcal{C}^{j+1})\le
\delta(\mathbf{t};\mathcal{C}^j)$ for all $\mathbf{t}$; the max over $X$ preserves this, and iterating gives the second claim. 
\end{proof}

\begin{theorem}[Exact budget and coverage]
For any budget $B$ with $B<M$ and $|\mathcal{C}_{\text{init}}|\le B$, CoVeR returns $\mathcal{C}$ with
\[
|\mathcal{C}|=B,\quad d_H(X,\mathcal{C})\le d_H(X,\mathcal{C}_{\text{init}})\le\sqrt{3}\,v_s
\]
\end{theorem}
\begin{proof}
Stage 2 runs $B_{\text{expan}}=B-|\mathcal{C}_{\text{init}}|\ge0$
steps. Selected indices are marked $d_{\min}=-1$, and the update
$d_{\min}(m)\leftarrow\min(d_{\min}(m),\|\cdot\|_2^2)$ preserves the mark. Before step $j$, $|\mathcal{C}_{\text{init}}|+(j-1)\le B-1<M$, so an unselected index exists and the $\arg\max$ returns it. Each step adds a new token, giving $|\mathcal{C}|= |\mathcal{C}_{\text{init}}| + B_{\text{expan}}= B$. 

\noindent\emph{Coverage.} Chain Lemma 1 and Lemma 2.
\end{proof}

\begin{lemma}[Greedy pick realizes $d_H$]
$\delta(\mathbf{t}_{m_{j+1}^*};\mathcal{C}^j)=d_H(X,\mathcal{C}^j)$.
\end{lemma}
\begin{proof}
Stage 2 selects $m_{j+1}^*=\arg\max_{m\notin\mathcal{C}^j}
\delta(\mathbf{t}_m;\mathcal{C}^j)^2$ (i.e., $\mathcal{D}_{\text{spatial}}$); as $u\mapsto u^2$ is increasing on
$[0,\infty)$, this equals the $\arg\max$ of $\delta$. Since
$\delta(\mathbf{t};\mathcal{C}^j)=0$ for $\mathbf{t}\in\mathcal{C}^j$, the maximum of $\delta(\cdot;\mathcal{C}^j)$ over $X$ equals its maximum over
$X\setminus\mathcal{C}^j$ whenever $d_H(X,\mathcal{C}^j)>0$, and both equal $0$ otherwise.
\end{proof}

\begin{proposition}[2-approximation at the expansion budget]
Let $\mathrm{OPT}_k(X)=\min_{|\mathcal{A}|=k}d_H(X,\mathcal{A})$ and $r=d_H(X,\mathcal{C})$. If $B_{\text{expan}}\ge1$, then $d_H(X,\mathcal{C})\le 2\cdot\mathrm{OPT}_{B_{\text{expan}}}(X)$.
\end{proposition}
\begin{proof}
Let $R^{j}=d_H(X,\mathcal{C}^{j})$; by Lemma~2, $R^{0}\ge\cdots\ge R^{B_{\text{expan}}}=r$. If $r=0$, the claim is immediate, so
assume $r>0$. For $1\le i<j\le B_{\text{expan}}$,
$m_i^*\in\mathcal{C}^{j-1}$, so by Lemma~3,
$\|\mathbf{t}_{m_j^*}-\mathbf{t}_{m_i^*}\|_2\ge\delta(\mathbf{t}_{m_j^*};\mathcal{C}^{j-1})=R^{j-1}\ge r$.
Pick $\mathbf{x}^*\in X$ with $\delta(\mathbf{x}^*;\mathcal{C})=r$ (unselected, since $r>0$); it is $\ge r$ from
every selected token, so $P=\{\mathbf{t}_{m_1^*},\dots,\mathbf{t}_{m_{B_{\text{expan}}}^*},\mathbf{x}^*\}$ is a set of $B_{\text{expan}}+1$ points pairwise $\ge r$ apart. Let $\mathcal{A}$ be any index set with
$|\mathcal{A}|=B_{\text{expan}}$ and $\rho=d_H(X,\mathcal{A})$. Since $|P|>|\mathcal{A}|$, two points
$\mathbf{p},\mathbf{q}\in P$ share a nearest index $a\in\mathcal{A}$, so by the triangle inequality $r\le\|\mathbf{p}-\mathbf{q}\|_2\le\|\mathbf{p}-\mathbf{t}_a\|_2+\|\mathbf{t}_a-\mathbf{q}\|_2\le2\rho$. Thus, minimizing over $\mathcal{A}$ gives $r\le2\cdot\mathrm{OPT}_{B_{\text{expan}}}(X)$.
\end{proof}

\noindent\textbf{Remark 1 (coverage interpretation).} Proposition~1 bounds the final selection $\mathcal{C}$, but against $\mathrm{OPT}_{B_{\text{expan}}}$ rather than $\mathrm{OPT}_B$, since the
$B_{\text{init}}$ voxel seeds need not be mutually separated and the pairwise argument uses the $B_{\text{expan}}$ FPS picks together with the worst-covered point. Combined with Theorem~1, CoVeR attains:
\[
d_H(X,\mathcal{C})\le\min\!\big(\sqrt{3}\,v_s,\; 2\cdot\mathrm{OPT}_{B_{\text{expan}}}(X)\big)
\]
These bounds provide complementary guarantees on the final coverage without solving the NP-hard $k$-center objective.

\begin{figure*}[t]
\begin{center}
\includegraphics[width=\textwidth]{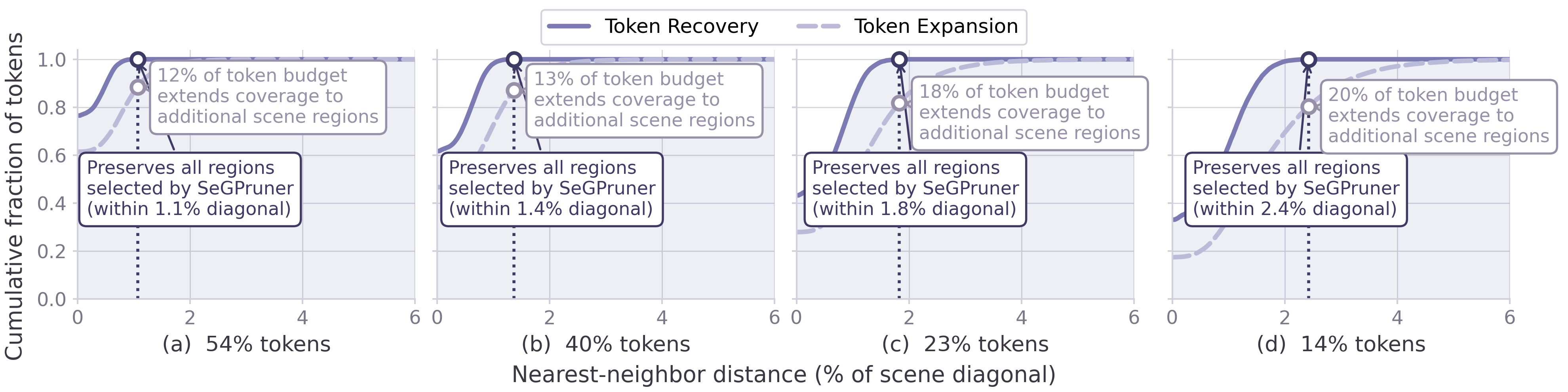}
\caption{\textbf{Cumulative distributions of directed
nearest-neighbor distances on ScanQA} across multiple token budgets. At all budgets, the Token Recovery (TR) reaches one, which indicates CoVeR keeps a token near
every region SeGPruner selects, within $1.1\%$, $1.4\%$, $1.8\%$, and $2.4\%$ of the scene diagonal, respectively. The Token Expansion (TE) also shows that CoVeR additionally covers regions SeGPruner leaves unrepresented.}
\label{fig:recovery_expansion_supp}
\end{center}
\vspace{-1.5em}
\end{figure*}

\begin{table}[t]
\centering
\caption{\textbf{Category-level performance on OpenEQA.} We compare CoVeR with VTC~\cite{dtc} and DTC~\cite{dtc} based on their published scores, computed by \texttt{GPT-4}~\cite{gpt4}. \textbf{LLM-Match.} is the overall score across seven categories and \textbf{Rel.} is average percentage of performance maintained. Categories: (a) object recognition, (b) object localization, (c) attribute recognition, (d) spatial understanding, (e) object state recognition, (f) functional reasoning, and (g) world knowledge. \texttt{*LLaVA-OV-7B}}
\label{tab:openeqa_gpt4}
\renewcommand{\arraystretch}{1.1}
\setlength{\tabcolsep}{5pt}
\resizebox{0.48\textwidth}{!}{
\begin{tabular}{l | c c c c c c c |c c}
\toprule
\multirow{3}{*}{Methods} & \multicolumn{7}{c|}{EQA Category} & \multirow{3}{*}{\makecell{LLM\\Match.}} & \multirow{3}{*}{Rel.} \\
\cmidrule(lr){2-8}
 & (a) & (b) & (c) & (d) & (e) & (f) & (g) &  & \\
\midrule
\multicolumn{10}{c}{Retain 100\% Tokens} \\
\midrule
\texttt{LLaVA*} & 48.6 & 43.0 & 74.4 & 43.6 & 74.5 & 53.1 & 55.0 & 56.2 & 100\\

\midrule
\multicolumn{10}{c}{Retain 43\% Tokens} \\
\midrule
\textcolor{ExpCText}{VTC} & -- & -- & -- & -- & -- & -- & -- & 54.2 & 96.4\\
\textcolor{ExpCText}{DTC} & -- & -- & -- & -- & -- & -- & -- & 54.3 & 96.6 \\
\textbf{CoVeR} & 49.9 & 40.3 & 74.6 & 45.2 & 71.6 & 54.0 & 54.6 & \cellcolor{tablecolor}\textbf{55.9} & \cellcolor{tablecolor}\textbf{99.5} \\

\midrule
\multicolumn{10}{c}{Retain 26\% Tokens} \\
\midrule
\textcolor{ExpCText}{VTC} & 40.3 & 35.7 & 62.9 & 40.3 & 71.5 & 52.5 & 49.5 & 50.5 & 89.9\\
\textcolor{ExpCText}{DTC} & 44.6 & 39.4 & 72.7 & 43.9 & 71.3 & 53.0 & 53.2 & 54.1 & 96.3\\
\textbf{CoVeR} & 47.9 & 40.0 & 72.8 & 45.1 & 73.0 & 54.2 & 54.6 & \cellcolor{tablecolor}\textbf{55.5} & \cellcolor{tablecolor}\textbf{98.8} \\

\midrule
\multicolumn{10}{c}{Retain 17\% Tokens} \\
\midrule
\textcolor{ExpCText}{VTC} & 37.2 & 33.1 & 59.9 & 38.4 & 65.9 & 51.0 & 45.9 & 47.4 & 84.3\\
\textcolor{ExpCText}{DTC} & 41.2 & 38.1 & 68.3 & 42.3 & 70.5 & 55.3 & 50.7 & 52.5 & 93.4\\
\textbf{CoVeR} & 46.7 & 40.1 & 68.4 & 45.1 & 69.9 & 54.3 & 53.3 & \cellcolor{tablecolor}\textbf{54.0} & \cellcolor{tablecolor}\textbf{96.1} \\

\midrule
\multicolumn{10}{c}{Retain 8\% Tokens} \\
\midrule
\textcolor{ExpCText}{VTC} & -- & -- & -- & -- & -- & -- & -- & 43.6 & 77.6 \\
\textcolor{ExpCText}{DTC} & -- & -- & -- & -- & -- & -- & -- & 49.3 & 87.7 \\
\textbf{CoVeR} & 43.1 & 34.7 & 61.9 & 41.7 & 65.9 & 52.0 & 51.4 & \cellcolor{tablecolor}\textbf{50.1} & \cellcolor{tablecolor}\textbf{89.1} \\
\bottomrule
\end{tabular}
}
\end{table}

\noindent\textbf{Remark 2 (case 3).} If $|\mathcal{C}_{\text{init}}|>B$, CoVeR keeps the representatives of the $B$ most populated voxels (Sec.~\ref{sec:overview}). The budget is still exact, but this subset drops some voxel representatives, so Lemma 1 need not hold. This case does not arise at any budget or split we evaluate.

\noindent\textbf{Summary.} The analysis gives (1) an exact per-scene budget $|\mathcal{C}|=B$; (2) under the inactive-safeguard condition, every token within $\sqrt{3}\,v_s$ of a retained token after stage 1, preserved through stage 2; and (3) a $2$-approximation to the $k$-center optimum at the expansion budget. Stage 1 thus contributes the absolute $\sqrt{3}\,v_s$ bound at low cost, while stage 2 fills the rest with a 2-approximation.

Fig.~\ref{fig:appendix_b} validates the bound: $d_H(X,\mathcal{C})$ rises gradually as the budget shrinks
but stays well below $\sqrt{3}\,v_s$, showing the bound is conservative.

\begin{table}[!ht]
\centering
\caption{\textbf{Category-level performance on OpenEQA.} We compare CoVeR with VisPruner~\cite{vispruner} and SeGPruner~\cite{segpruner}, computed by \texttt{GPT-4o}~\cite{gpt4o}. \textbf{LLM-Match.} is the overall score across seven categories and \textbf{Rel.} is the average percentage of performance maintained. Categories: (a) object recognition, (b) object localization, (c) attribute recognition, (d) spatial understanding, (e) object state recognition, (f) functional reasoning, and (g) world knowledge. \texttt{*LLaVA-OV-7B}}
\renewcommand{\arraystretch}{1.1}
\label{tab:openeqa_gpt4o}
\setlength{\tabcolsep}{5pt}
\resizebox{.48\textwidth}{!}{
\begin{tabular}{l | c c c c c c c |c c}
\toprule
\multirow{3}{*}{Methods} & \multicolumn{7}{c|}{EQA Category} & \multirow{3}{*}{\makecell{LLM\\Match.}} & \multirow{3}{*}{Rel.} \\
\cmidrule(lr){2-8}
 & (a) & (b) & (c) & (d) & (e) & (f) & (g) &  & \\
\midrule

\multicolumn{10}{c}{Retain 100\% Tokens} \\
\midrule
\texttt{LLaVA*} & 53.9 & 43.4 & 77.6 & 49.1 & 73.1 & 58.9 & 57.0 & 59.1 & 100 \\

\midrule
\multicolumn{10}{c}{Retain 43\% Tokens} \\
\midrule
\textcolor{ExpDText}{VisPruner}
& 52.3 & 41.9 & 76.0 & 49.0 & 74.0 & 58.2 & 57.2
& 58.4 & 98.8 \\

\textcolor{ExpDText}{SeGPruner}
& 52.5 & 43.2 & 74.3 & 50.2 & 73.0 & 56.2 & 56.3
& 58.0 & 98.1 \\

\textbf{CoVeR}
& 52.5 & 41.2 & 77.1 & 51.3 & 71.9 & 58.9 & 57.2
& \cellcolor{tablecolor}\textbf{58.6}
& \cellcolor{tablecolor}\textbf{99.2} \\

\midrule
\multicolumn{10}{c}{Retain 26\% Tokens} \\
\midrule
\textcolor{ExpDText}{VisPruner}
& 51.0 & 39.2 & 75.5 & 51.4 & 70.8 & 56.8 & 55.4
& 57.1 & 96.6 \\

\textcolor{ExpDText}{SeGPruner}
& 49.9 & 39.8 & 76.2 & 50.1 & 72.6 & 57.4 & 56.7
& 57.5 & 97.3 \\

\textbf{CoVeR}
& 51.0 & 40.9 & 75.1 & 49.9 & 73.2 & 57.8 & 56.0
& \cellcolor{tablecolor}\textbf{57.7}
& \cellcolor{tablecolor}\textbf{97.6} \\

\midrule
\multicolumn{10}{c}{Retain 17\% Tokens} \\
\midrule
\textcolor{ExpDText}{VisPruner}
& 50.4 & 38.6 & 70.9 & 50.1 & 69.1 & 57.6 & 54.8
& 55.9 & 94.6 \\

\textcolor{ExpDText}{SeGPruner}
& 48.3 & 38.7 & 72.4 & 50.6 & 69.1 & 57.3 & 56.6
& 56.0 & 94.8 \\

\textbf{CoVeR}
& 49.7 & 41.4 & 72.0 & 50.5 & 70.3 & 58.3 & 55.4
& \cellcolor{tablecolor}\textbf{56.8}
& \cellcolor{tablecolor}\textbf{96.1} \\

\midrule
\multicolumn{10}{c}{Retain 8\% Tokens} \\
\midrule
\textcolor{ExpDText}{VisPruner}
& 40.5 & 36.9 & 65.5 & 45.5 & 64.6 & 56.7 & 51.4
& 51.5 & 87.1 \\

\textcolor{ExpDText}{SeGPruner}
& 44.6 & 38.6 & 63.4 & 46.6 & 65.8 & 55.0 & 53.5
& 52.5 & 88.8 \\

\textbf{CoVeR}
& 47.3 & 36.3 & 65.6 & 46.3 & 66.2 & 55.5 & 54.0
& \cellcolor{tablecolor}\textbf{53.0}
& \cellcolor{tablecolor}\textbf{89.7} \\
\bottomrule
\end{tabular}}
\vspace{-1em}
\end{table}

\begin{table*}[t]
\centering
\begin{minipage}[t]{0.477\textwidth}
\centering
\renewcommand{\arraystretch}{1.1}
\caption{\textbf{OpenEQA comparison.} CoVeR remains competitive,  relative to both commercial and open VLMs, after removing up to 92\% of visual tokens. \textcolor{gray}{($\cdot$)} denote change from the base model. LLM-Match uses \texttt{GPT-4}~\cite{gpt4}.}
\label{tab:comparison_vlm_openeqa}

\resizebox{\linewidth}{!}{
\setlength{\tabcolsep}{2pt}
\begin{tabular}{lc}
\toprule
\textbf{Models} & \textbf{LLM-Match$\uparrow$}\\
\midrule

\graydifftext{\textit{Blind Text-only LLM baseline}} \\
GPT4~\cite{gpt4} & 33.5 \\
LLaMA-2 70B~\cite{llama2} & 28.3 \\
\midrule

\graydifftext{\textit{Closed VLM models}} \\
Claude-3 Opus & 36.3 \\
Gemini 1.0 Pro Vision~\cite{team2023gemini} & 44.9 \\
Claude-3.5 Sonnet & 48.7 \\
GPT4-V (15 frames) \cite{gpt4} & 54.6 \\
GPT4-V (50 frames) \cite{gpt4} & 55.3 \\
\midrule

\graydifftext{\textit{Open VLM models}} \\
Video-LLaMA~\cite{zhang2023video} {\scriptsize\textcolor{gray}{[EMNLP 2023]}} & 20.0 \\
 LLaMA-2 w/ Concept Graph~\cite{openeqa} {\scriptsize\textcolor{gray}{[CVPR 2024]}}  & 28.7 \\
AuroraCap~\cite{chai2025auroracap} {\scriptsize\textcolor{gray}{[ICLR 2025]}} & 28.9 \\
Video-ChatGPT~\cite{maaz2024video} {\scriptsize\textcolor{gray}{[ACL 2024]}} & 32.1 \\
LLaMA-2 w/ Sparse Voxel Map~\cite{openeqa} {\scriptsize\textcolor{gray}{[CVPR 24]}} & 34.3 \\
LLaMA-2 w/ LLaVA-1.5 Caption~\cite{openeqa} {\scriptsize\textcolor{gray}{[CVPR 24]}} & 36.8 \\
Chat-UniVi~\cite{jin2024chat} {\scriptsize\textcolor{gray}{[CVPR 2024]}} & 42.3 \\
Video-LLaMA2~\cite{videollama} {\scriptsize\textcolor{gray}{[arXiv 2024]}} & 49.2 \\
InternVL2.5 (16 frames)~\cite{internvl25} {\scriptsize\textcolor{gray}{[arXiv 2024]}} & 54.4 \\
MovieChat (w/ \texttt{LLaVA-OV-7B})~\cite{moviechat} {\scriptsize\textcolor{gray}{[CVPR 2024]}}  & 54.9 \\
LLaVA-3D (32 frames)~\cite{llava3d} {\scriptsize\textcolor{gray}{[ICCV 2025]}} & 53.2 \\

Qwen2.5-VL (16 frames)~\cite{qwen2.5} {\scriptsize\textcolor{gray}{[arXiv 2025]}} & 50.8 \\
Magma (16 frames)~\cite{magma} {\scriptsize\textcolor{gray}{[CVPR 2025]}} & 49.1 \\
NVILA (16 frames)~\cite{nvila} {\scriptsize\textcolor{gray}{[CVPR 2025]}} & 54.0 \\
InternVL3 (16 frames)~\cite{internvl3} {\scriptsize\textcolor{gray}{[arXiv 2025]}} & 55.5 \\
ThinkAct (16 frames)~\cite{thinkact} {\scriptsize\textcolor{gray}{[NeurIPS 2026]}} & 56.2 \\
\midrule

\graydifftext{\textit{Ours} (12 frames)} \\
\covbar{100}{\texttt{LLaVA-OV-7B} (100\%)} & 56.2\\
\covbar{56}{w/ CoVeR} \tokdrop{44} & 56.0\graydiff{-0.2}\\
\covbar{43}{w/ CoVeR} \tokdrop{57} & 55.9\graydiff{-0.3}\\
\covbar{26}{w/ CoVeR} \tokdrop{74} & 55.5\graydiff{-0.7}\\
\covbar{17}{w/ CoVeR} \tokdrop{83} & 54.0\graydiff{-2.2}\\
\covbar{8}{w/ CoVeR} \tokdrop{92} & 50.1\graydiff{-6.1}\\

\bottomrule
\end{tabular}
}
\end{minipage}
\hfill
\begin{minipage}[t]{0.5065\textwidth}
\centering
\renewcommand{\arraystretch}{1}
\caption{\textbf{ScanQA (val) comparison.} CoVeR remains competitive with task-specific 3D models, video-LMMs, and fine-tuned 3D-LMMs, after removing up to 91\% of visual tokens, without 3D-specific training or learned pruning.}
\label{tab:comparison_vlm_scanqa}
\resizebox{\linewidth}{!}{
\begin{tabular}{lc}
\toprule
\textbf{Models} & \textbf{EM@1$\uparrow$}\\
\midrule
\graydifftext{\textit{Task-specific models}} \\
VoteNet+MCAN~\cite{yu2019deep} {\scriptsize\textcolor{gray}{[CVPR 2019]}} & 17.3 \\
ScanRefer+MCAN~\cite{yu2019deep} {\scriptsize\textcolor{gray}{[CVPR 2019]}} & 18.6 \\
ScanQA~\cite{scanqa} {\scriptsize\textcolor{gray}{[CVPR 2022]}} & 21.1 \\
Jin et al.~\cite{jin2023context} {\scriptsize\textcolor{gray}{[CVPR 2023]}} & 21.7 \\
3D-VisTA~\cite{3d-vista} {\scriptsize\textcolor{gray}{[ICCV 2023]}} & 22.4 \\
3DVLP ~\cite{zhang2024vision} {\scriptsize\textcolor{gray}{[AAAI 2024]}} & 24.0 \\
DSPNet (20 frames) \cite{dspnet} {\scriptsize\textcolor{gray}{[CVPR 2025]}} & 23.5 \\
\midrule

\graydifftext{\textit{Video-LMMs}} \\
Agent3D-Zero~\cite{zhang2024agent3d} {\scriptsize\textcolor{gray}{[ECCV 2024]}} & 17.5 \\
LLaVA-NeXT-Video~\cite{liu2024llavanext} {\scriptsize\textcolor{gray}{[arXiv 2024]}} & 18.7 \\
MovieChat (w/ LLaVA-OV-7B)~\cite{moviechat} {\scriptsize\textcolor{gray}{[CVPR 24]}}  & 26.0 \\
AuroraCap~\cite{chai2025auroracap} {\scriptsize\textcolor{gray}{[ICLR 2025]}} & 17.2 \\
\midrule

\graydifftext{\textit{Task-specific fine-tuned image/video LMMs}} \\
NaviLLM~\cite{zheng2024towards} {\scriptsize\textcolor{gray}{[CVPR 2024]}} & 23.0\\
Spatial-MLLM-4B~\cite{spatialllm} {\scriptsize\textcolor{gray}{[NeurIPS 2025]}} & 26.3\\
SPAR-mix \cite{sparmix} {\scriptsize\textcolor{gray}{[NeurIPS 2025]}} & 27.7 \\
SplatTalk-ScanQA-FT (100 frames)~\cite{splattalk} {\scriptsize\textcolor{gray}{[ICCV 2025]}}  & 22.3 \\
Proxy3D (32 frames)~\cite{proxy3d} {\scriptsize\textcolor{gray}{[CVPR 2026]}} & 25.2\\
\midrule
\graydifftext{\textit{Task-specific fine-tuned 3D-LMMs}} \\
3D-LLM~\cite{3dllm} {\scriptsize\textcolor{gray}{[NeurIPS 2023]}} & 20.5 \\
FE-3DGQA~\cite{zhao2022toward} {\scriptsize\textcolor{gray}{[TCSVT 2022]}} & 22.3 \\
ChatScene~\cite{chatscene} {\scriptsize\textcolor{gray}{[NeurIPS 2024]}} & 21.6 \\
LEO~\cite{leo} {\scriptsize\textcolor{gray}{[ICML 2024]}} & 24.5 \\
Scene-LLM~\cite{scenellm} {\scriptsize\textcolor{gray}{[WACV 2025]}} & 27.2 \\
Yuan et al. \cite{yuan2025empowering} {\scriptsize\textcolor{gray}{[CVPR 2025]}} & 22.9 \\
Inst3D-LMM \cite{inst3d} {\scriptsize\textcolor{gray}{[CVPR 2025]}} & 24.6 \\
\midrule

\graydifftext{\textit{Ours} (12 frames)} \\
\covbar{100}{\texttt{LLaVA-OV-7B} (100\%)} & 28.2\\
\covbar{54}{w/ CoVeR} \tokdrop{46} & 28.7\graydiff{+0.5}\\
\covbar{40}{w/ CoVeR} \tokdrop{60} & 28.9\graydiff{+0.7}\\
\covbar{23}{w/ CoVeR} \tokdrop{77} & 28.5\graydiff{+0.3}\\
\covbar{14}{w/ CoVeR} \tokdrop{86} & 27.9\graydiff{-0.3}\\
\covbar{9}{w/ CoVeR} \tokdrop{91} & 27.1\graydiff{-1.1}\\

\bottomrule
\end{tabular}
}
\end{minipage}
\end{table*}

\section{Results on Token Recovery and Expansion}

Figure \ref{fig:recovery_expansion_supp} shows the cumulative distributions on Token Recovery (TR) and Token Expansion (TE) at all budget levels. At every budget, the TR curve reaches one, so for every region SeGPruner selects, CoVeR retains a token nearby, within 2.4\% of the scene diagonal at the tightest budget. The TE further shows that CoVeR also places tokens in regions SeGPruner leaves uncovered. By prioritizing coverage, CoVeR still retains the informative regions that \textcolor{ExpDText}{\textit{learned importance}} methods select, while additionally covering other regions in the scene. This shows our advantage over \textcolor{ExpDText}{\textit{learned importance}}, which spends its budget on near-duplicate tokens from a few prominent regions and thus leaves the scene unrepresented.

\section{Additional Quantitative Results}
\noindent \textbf{OpenEQA Category Analysis.} CoVeR is particularly strong on spatial understanding (Tables~\ref{tab:openeqa_gpt4}-\ref{tab:openeqa_gpt4o}). At 17\% retention, its spatial score exceeds the full model when evaluated via \texttt{GPT-4} LLM Match (45.1 vs. 43.6) and \texttt{GPT-4o} (50.5 vs. 49.1). At 8\%, it also gives the best object-recognition score (47.3 vs. 44.6 for SeGPruner and 40.5 for VisPruner), with 89.1-89.7\% overall retention across both judges. 

\noindent \textbf{Other 3D Reasoning Models.} Tables~\ref{tab:comparison_vlm_openeqa} and~\ref{tab:comparison_vlm_scanqa} provide a broader context across task-specific 3D models, video LMMs, and general VLMs. CoVeR reaches 50.1 OpenEQA LLM-Match with 8\% of tokens and loses only 0.7 points at 26\% while outperforming prior models. On ScanQA, 23\% retention yields 28.5 EM@1, exceeding the listed task-specific and open VLM systems. Thus, the pruned model remains competitive in absolute terms, not only relative to its full-token \texttt{LLaVA-OV-7B} baseline.
    
\section{Additional Qualitative Results}
Figures~\ref{fig:scanqa_vis1}, \ref{fig:scanqa_vis2}, \ref{fig:sqa_vis1}, \ref{fig:sqa_vis2}, and \ref{fig:openeqa_vis} provide additional qualitative results from ScanQA, SQA3D, and OpenEQA datasets. Across different token budgets and 3D reasoning tasks, CoVeR preserves sufficient visual evidence to support object understanding, spatial and situated reasoning, and embodied question answering despite significant token pruning. These visualizations complement our quantitative findings by demonstrating that coverage-based token pruning maintains diverse scene information to perform 3D multi-view reasoning.

\begin{figure*}[t]
    \centering
    \includegraphics[width=0.95\textwidth]{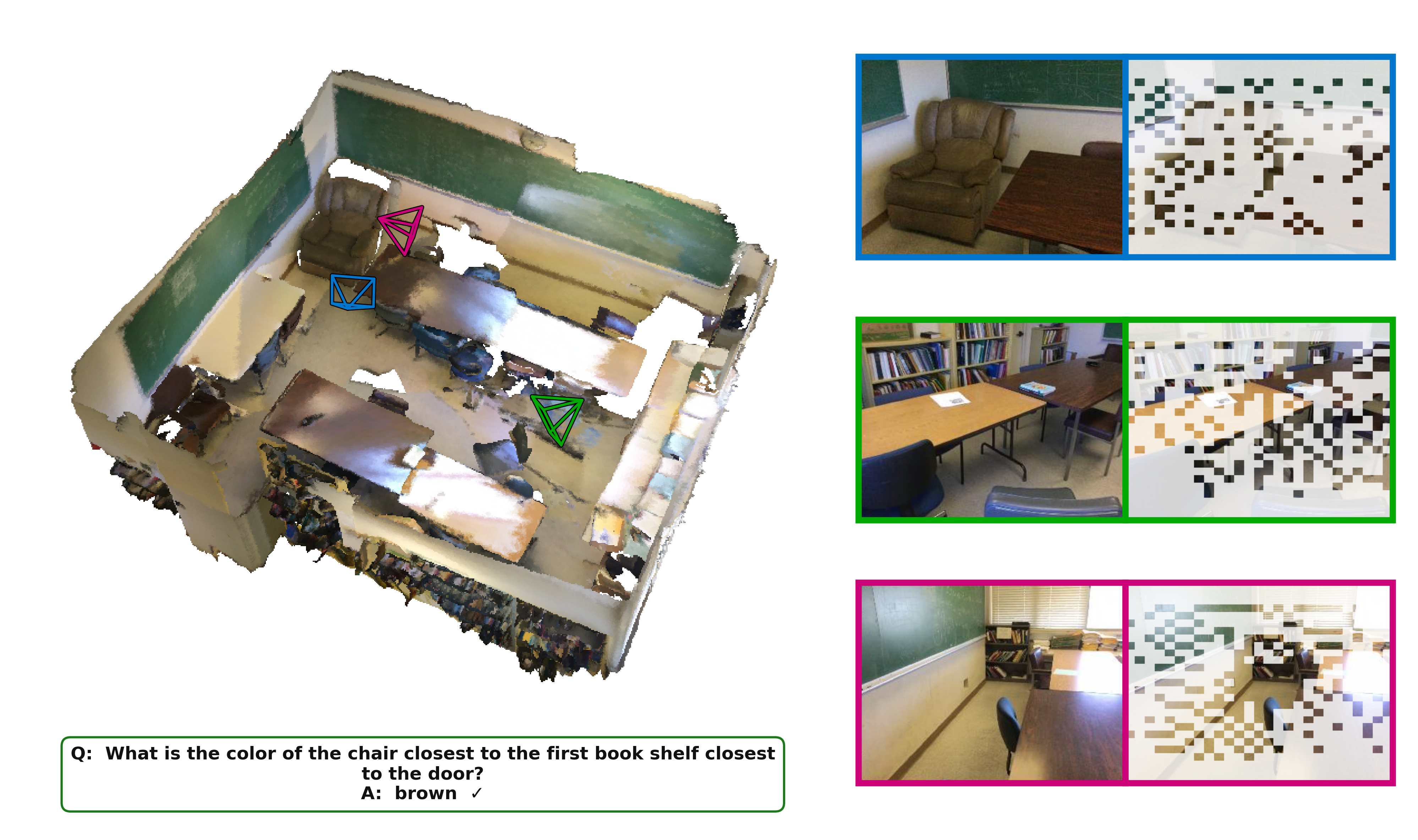}
    \caption{\textbf{Visual Results on ScanQA}~\cite{scanqa}
    at 23\% token retention.}
    \label{fig:scanqa_vis1}
\end{figure*}

\begin{figure*}[t]
    \centering
    \includegraphics[width=0.95\textwidth]{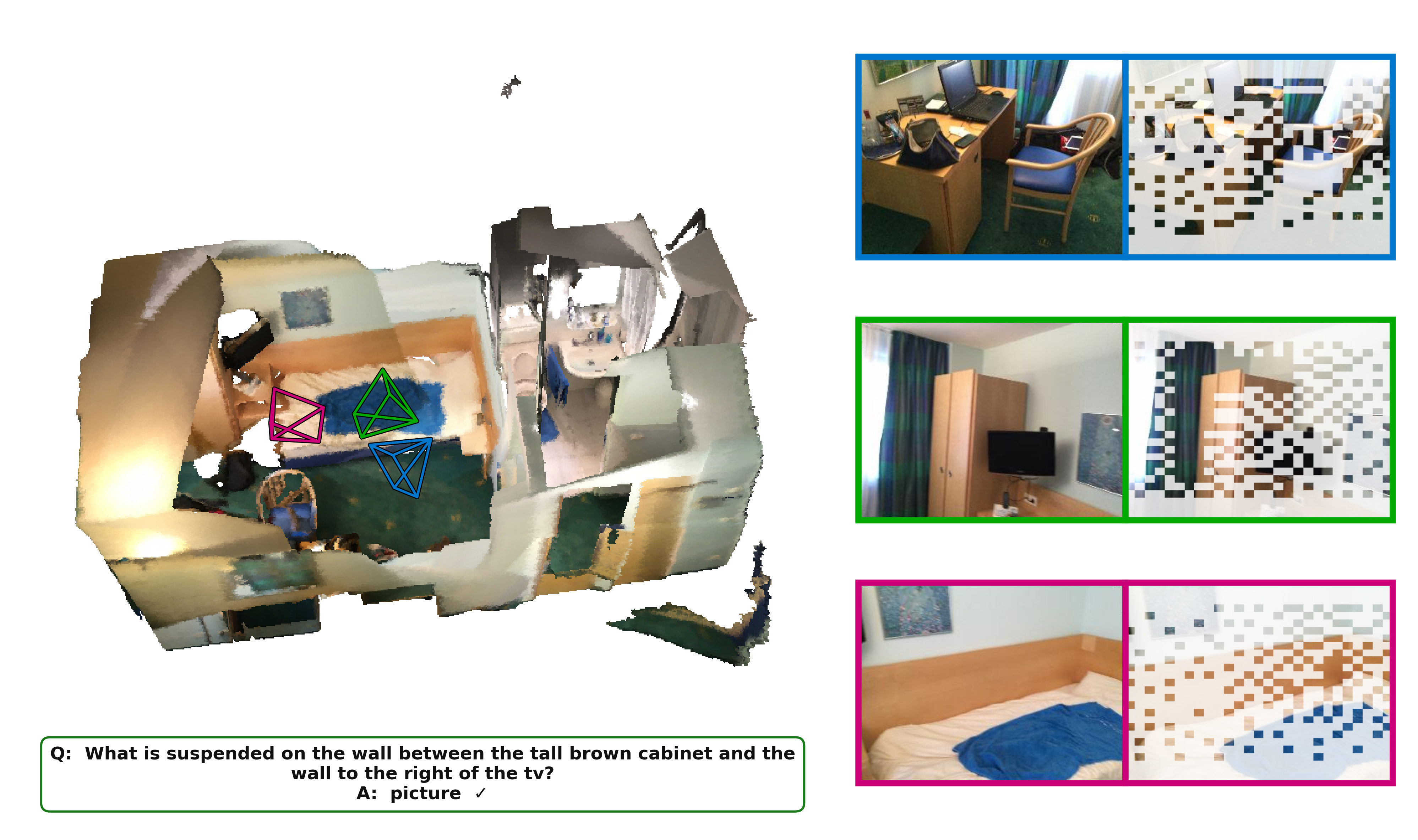}
    \caption{\textbf{Visual Results on ScanQA}~\cite{scanqa}
    at 23\% token retention.}
    \label{fig:scanqa_vis2}
\end{figure*}

\begin{figure*}[t]
    \centering
    \includegraphics[width=0.9\textwidth]{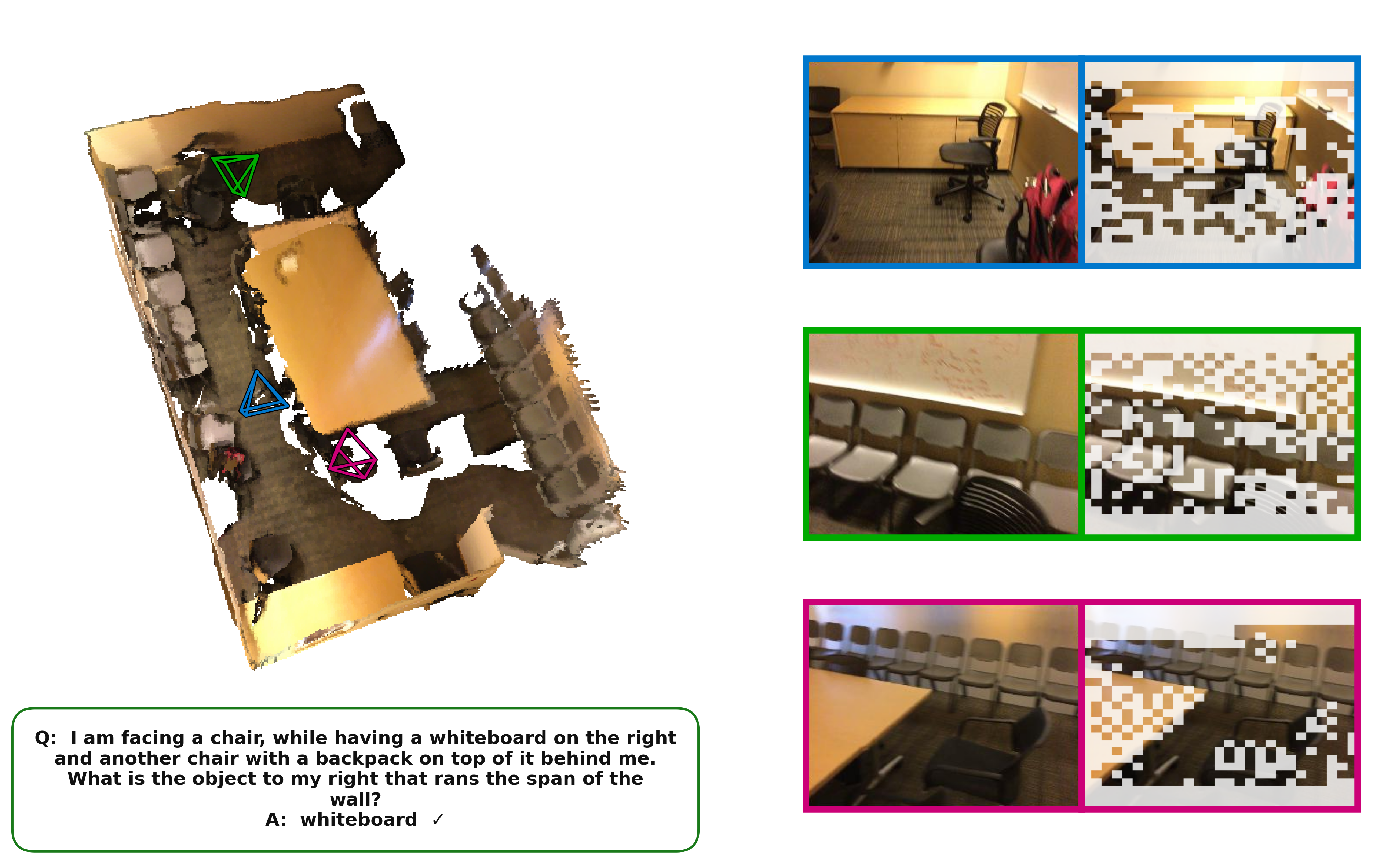}
    \caption{\textbf{Visual Results on SQA3D}~\cite{sqa}
    at 43\% token retention.}
    \label{fig:sqa_vis1}
\end{figure*}

\begin{figure*}[t]
    \centering
    \includegraphics[width=0.9\textwidth]{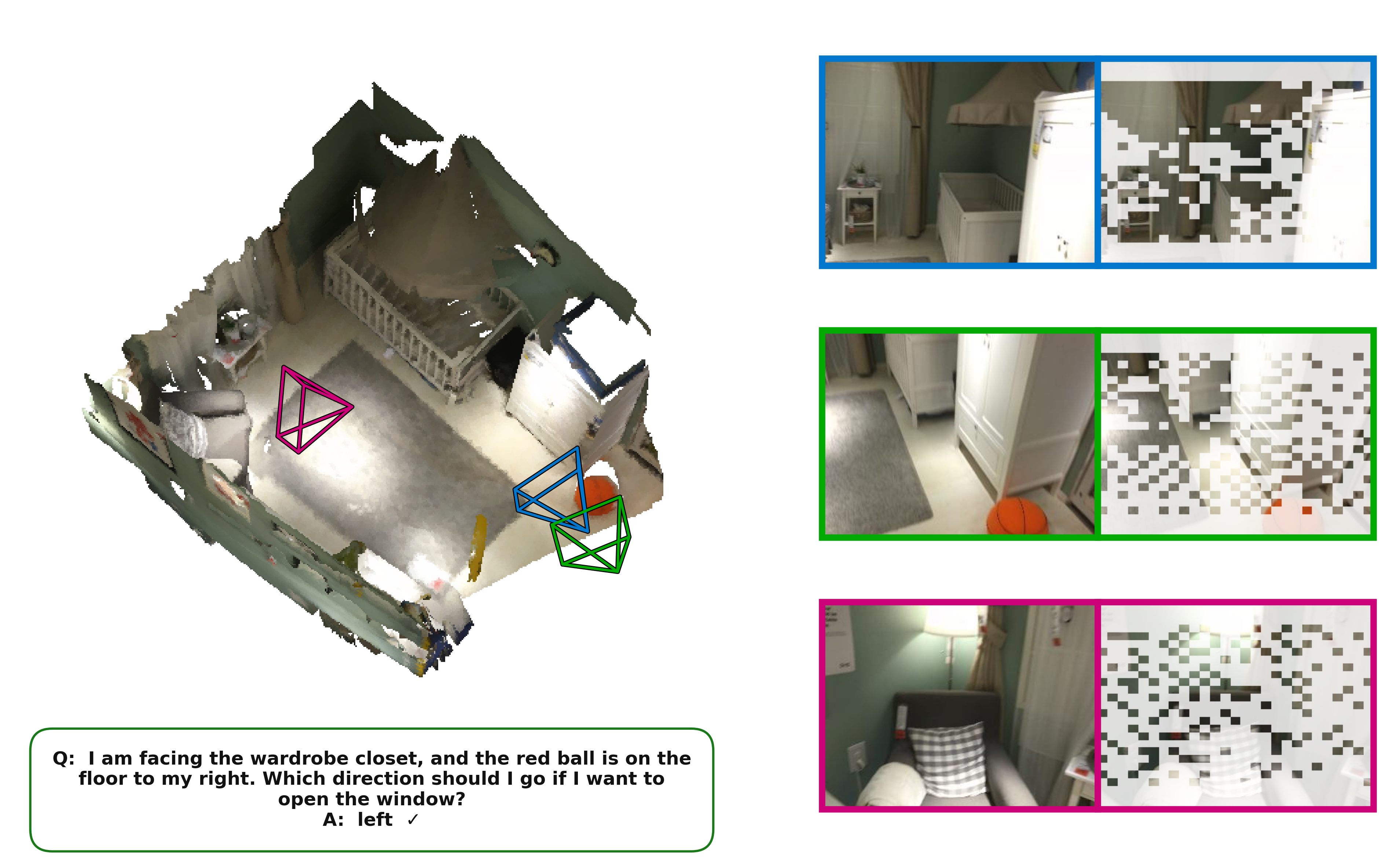}
    \caption{\textbf{Visual Results on SQA3D}~\cite{sqa}
    at 26\% token retention.}
    \label{fig:sqa_vis2}
\end{figure*}

\begin{figure*}[t]
    \centering
    \includegraphics[width=0.95\textwidth]{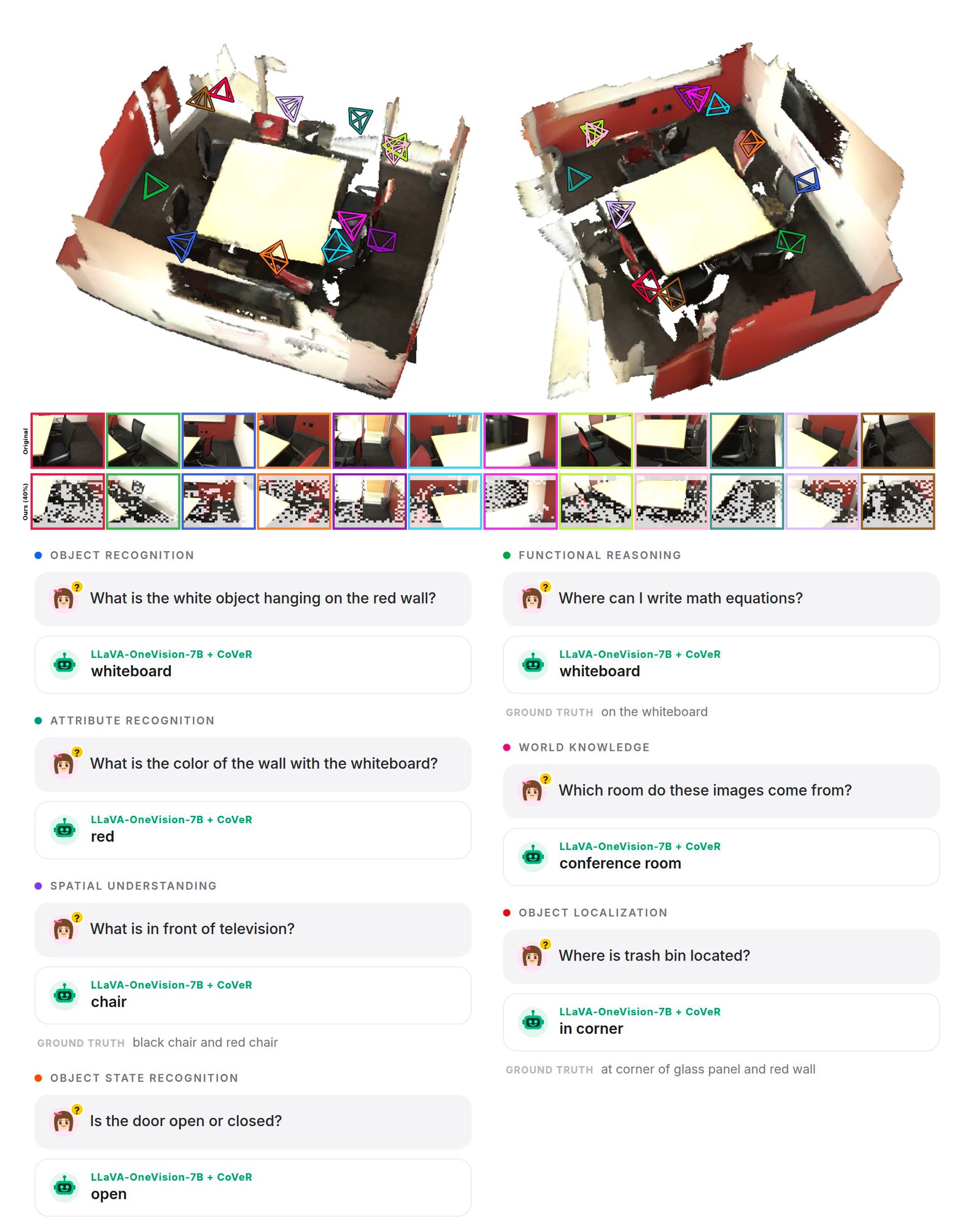}
    \caption{\textbf{Visual Results on OpenEQA}~\cite{openeqa}
    at 40\% token retention.}
    \label{fig:openeqa_vis}
\end{figure*}

\section{Broader Impact}
\textbf{Positive impact.} CoVeR significantly reduces inference computation and memory, while maintaining the performance of baseline VLMs, making multi-view 3D reasoning more accessible to users with limited computational resources. It also provides precise token reduction control in VLMs. CoVeR also potentially lowers deployment energy costs in robotic and embodied AI applications.

\noindent\textbf{Potential negative impact.} Real-world visual reasoning can raise important privacy concerns. In addition, aggressive pruning may discard important cues, potentially leading to incorrect VLM answers in safety-critical systems like healthcare or autonomous driving. Besides, since CoVeR is a general-purpose method, it may inherit the societal risks, biases, and failure modes of its underlying VLMs and data.